\documentclass{IEEEtaes}

\usepackage{color,array,amsthm}
\usepackage{graphicx}

\usepackage{amsmath,amsfonts}
\usepackage{algorithmic}
\usepackage{algorithm}
\usepackage{array}
\usepackage[caption=false,font=normalsize,labelfont=sf,textfont=sf]{subfig}
\usepackage{textcomp}
\usepackage{stfloats}
\usepackage{url}
\usepackage{verbatim}
\usepackage{graphicx}
\usepackage{cite}
\usepackage{amssymb}
\usepackage{color}
\usepackage{amsthm, amssymb}

\jvol{XX}
\jnum{XX}
\jmonth{XXXXX}
\paper{1234567}
\pubyear{2020}
\doiinfo{TAES.2020.Doi Number}

\newtheorem*{remark}{Remark}
\newtheorem{assumption}{Assumption}
\newtheorem{theorem}{Theorem}
\newtheorem{lemma}{Lemma}
\begin{document}

\title{Formation Control for Enclosing and Tracking via Relative Localization}

\author{XUEMING LIU}
\author{DENGYU ZHANG}
\author{QINGRUI ZHANG}
\author{TIANJIANG HU}
\affil{Sun Yat-sen University, Shenzhen, P.R. China}


\receiveddate{ 
This work is supported by the Key-Area Research and Development Program of Guangdong Province under Grant 2024B1111060004, in part by the Basic and Applied Basic Research Foundation of Guangdong Province under Grant 2024A1515012408, in part by the National Natural Science Foundation of China under Grant 62103451 and 62473390, and in part by the Shenzhen Science and Technology Program JCYJ20220530145209021.}


\corresp{{\itshape (Corresponding author: Qingrui Zhang and Tianjiang Hu)}}

\authoraddress{ All authors are with the School of Aeronautics and Astronautics, Sun Yat-sen University (Shenzhen Campus), Shenzhen 518107, P.R. China
(e-mail: \{liuxm93, zhangdy56\}@mail2.sysu.edu.cn, \{zhangqr9, hutj3\}@mail.sysu.edu.cn).}


\markboth{XUEMING ET AL.}{FORMATION CONTROL FOR ENCLOSING AND TRACKING VIA RELATIVE LOCALIZATION}
\maketitle

\begin{abstract} This paper proposes an integrated framework for coordinating multiple unmanned aerial vehicles (UAVs) in a distributed manner to persistently enclose and track a moving target without relying on external localization systems. The proposed framework consists of three modules: cooperative state estimators, circular formation pattern generators, and formation tracking controllers. In the cooperative state estimation module, a recursive least squares estimator (RLSE) for estimating the relative positions between UAVs is integrated with a distributed Kalman filter (DKF), enabling a persistent estimation of the target's state. When a UAV loses direct measurements of the target due to environmental occlusion, measurements from neighbors are aligned into the UAV's local frame to provide indirect measurements. The second module focuses on planning a desired circular formation pattern using a coupled oscillator model. This pattern ensures an even distribution of UAVs around a circle that encloses the moving target. The persistent excitation property of the circular formation is crucial for achieving convergence in the first module. Finally, a consensus-based formation controller is designed to enable multiple UAVs to asymptotically track the planned circular formation pattern while ensuring bounded control inputs. Theoretical analysis demonstrates that the proposed framework ensures asymptotic tracking of a target with constant velocity. For a target with varying velocity, the tracking error converges to a bounded region related to the target's maximum acceleration. Simulations and experiments validate the effectiveness of the proposed algorithm.

  \end{abstract}
  
  
  \section{INTRODUCTION}

  Formation control of collective unmanned aerial vehicles (UAVs) has been extensively investigated in recent years due to its great advantages and potential in diverse missions \cite{chungSurveyAerialSwarm2018,zhang2017aerodynamics,zhang2021robust,DYZhang_IROS2023,CHYu_ICRA2024}. Among the various formation control applications, target tracking has garnered significant attention in areas such as persistent surveillance \cite{rajaEfficientFormationControl2021}, wildfire monitoring \cite{bayramTrackingWildlifeMultiple2017}, autonomous aerial filming \cite{kratkyAutonomousAerialFilming2021}, among others \cite{ZZhang_RAL2024,ZZhang_TNNLS2025}. In particular, circular formation, or circumnavigation, is a specialized pattern widely used for target enclosing and tracking \cite{litimeinSurveyTechniquesCircular2021,brinon-arranzCooperativeControlDesign2014}. This formation requires UAVs to be evenly distributed around the target at a specified distance.
  
  To accomplish the task of target enclosing and tracking, UAVs first need to estimate the target's state. Conventional approaches to cooperative state estimation primarily rely on the framework of distributed Kalman filters (DKFs) \cite{liDistributedKalmanFilter2020,liDistributedKalmanFilter2021,olfati-saberCollaborativeTargetTracking2011,lianDistributedKalmanConsensus2022,wangCooperativeTargetTracking2012,olfati-saberDistributedKalmanFiltering2007,liSelfLocalizationDistributedSystems2022,doostmohammadianDistributedEstimationApproach2022,xuDistributedPseudolinearEstimation2017}. These methods are considered to improve estimation accuracy and target tracking performance. A common assumption in these approaches is that the position of all UAVs, or at least a subset of them, can be obtained in a global frame through an external localization system like the Global Positioning System (GPS). However, obtaining reliable GPS signals can be difficult or even impossible in some scenarios \cite{jainEncirclementMovingTargets2022,dongCoordinateFreeCircumnavigationMoving2022}. Although some efforts have been made to address GPS-denied environments \cite{liuMovingTargetCircumnavigationUsing2023,dongCoordinateFreeCircumnavigationMoving2022,liVGSwarmVisionBasedGene2023,jainEncirclementMovingTargets2022,liuFormationControlMoving2023,yuBearingonlyCircumnavigationControl2019}, these works primarily focus on target tracking in open or unobstructed environments and do not account for situations where environmental factors obscure the target. Thus, achieving consistent cooperative state estimation for target tracking remains challenging without an external localization system, particularly in cases of partial measurement loss.
  
  To address this issue, this paper proposes an intuitive approach: incorporating the relative position information between UAVs into the traditional DKF method. In particular, the impact of occlusion can be mitigated by coordinating a team of UAVs at different locations \cite{ibenthalLocalizationPartiallyHidden2023}. Thus, UAV $i$ can derive an indirect measurement of the target by combining its relative position with UAV $j$ and the direct measurement of the target obtained by $j$. However, in the absence of external positioning systems, it is difficult to directly obtain the relative position of neighboring UAVs \cite{nguyenPersistentlyExcitedAdaptive2020}. To this end, several researchers have proposed different relative position estimators by integrating various relative measurements between UAVs \cite{liuMovingTargetCircumnavigationUsing2023,yuBearingonlyCircumnavigationControl2019,nguyenPersistentlyExcitedAdaptive2020,guoUltraWidebandOdometryBasedCooperative2020}, such as relative displacement, relative bearing, and relative distance. To ensure the convergence of these algorithms, all these studies emphasize that the measurement vector must meet persistent excitation conditions. However, these studies only treat the persistent excitation condition as a prerequisite assumption, and how to ensure the validity of this assumption is rarely discussed. Therefore, designing a formation control scheme that guarantees estimator convergence while effectively achieving the target enclosure and tracking task remains a significant challenge.

 To this end, this paper develops a variant of the coupled oscillator-based formation pattern design method from our previous work \cite{liuFormationControlMoving2023}. This approach is inspired by the stabilization problem in phase oscillators \cite{kleinIntegrationCommunicationControl2007,sepulchreStabilizationPlanarCollective2007,sepulchreStabilizationPlanarCollective2008}, specifically, the $(m,n)$-pattern problem \cite{dorflerSynchronizationComplexNetworks2014}. Unlike \cite{liuFormationControlMoving2023}, both the desired relative position and the desired relative velocity are designed, enhancing the performance of the tracking task. Most importantly, this time-varying circular formation inherently exhibits the property of persistent excitation. Therefore, using the desired relative state as input, a consensus-based formation controller is proposed to achieve target enclosing and tracking while ensuring the convergence of the relative state estimator. Compared to \cite{liuFormationControlMoving2023}, this paper incorporates measurement noise and discusses its impact on formation control. Additionally, in \cite{liuFormationControlMoving2023}, each UAV independently estimates the target's state and uses this estimate directly for control. In contrast, this paper employs a consensus-based DKF method to align the state estimates across all UAVs, ensuring consistency and reducing tracking errors.

 While most literature on target tracking typically assumes the target is stationary, moving at a constant velocity, or drifting slowly\cite{jainEncirclementMovingTargets2022,jainEncirclementMovingTargets2019,yangObservabilityEnhancementBoresightCalibration2023,liLocalizationCircumnavigationMultiple2018,doostmohammadianDistributedEstimationApproach2022,dongCoordinateFreeCircumnavigationMoving2022,shamesCircumnavigationUsingDistance2012a,xuDistributedPseudolinearEstimation2017}, this paper examines the formation controller's ability to track a target with variable velocity. Theoretical analysis indicates that the tracking error converges to a bounded value, which is positively correlated with the target's acceleration. These insights aid in the analysis and design of tracking algorithms. 
  To summarize, the main contributions are highlighted as follows.     \begin{enumerate}
      \item A relative state estimation framework for target tracking in the absence of external localization systems is proposed by combining a recursive least squares estimator (RLSE) from \cite{nguyenPersistentlyExcitedAdaptive2020} and a classical DKF from \cite{olfati-saberDistributedKalmanFiltering2007, liDistributedKalmanFilter2020, wangCooperativeTargetTracking2012}.
   
      \item A consensus-based formation controller for a second-order system is designed by integrating relative state estimates with a coupled oscillator-based formation pattern. This approach simultaneously guarantees the convergence of both the target tracking algorithm and the relative state estimators.
    
      \item Theoretical analysis demonstrates that the proposed framework ensures asymptotic stability for enclosing and tracking a target moving at a constant velocity without measurement noise. Furthermore, for a target with variable velocity, the tracking performance is related to the target's maximum acceleration, resulting in ultimately bounded tracking errors.
    \end{enumerate}

  The remainder of this paper is organized as follows. Section~\ref{prelimi} provides the fundamental definitions and outlines the problem. In Section~\ref{stateEsti}, the relative state estimation framework is proposed. Section~\ref{formationDesign} describes the coupled oscillator-based circular formation designed method for discrete second-order systems. The formation controller and theoretical analysis are given in Section~\ref{formationControl} and Section~\ref{Analysis} respectively. Numerical simulations and physical experiments are presented in Section~\ref{sim&exp}. At last, Section~\ref{conclu} gives the conclusions.

  \section{PRELIMINARIES} \label{prelimi}
  \emph{Notations:} The sets of natural numbers and real numbers are denoted by $\mathbb{N}$ and $\mathbb{R}$, respectively. For vector $\boldsymbol{x} \in \mathbb{R}^m$, $\|\boldsymbol{x}\|$ and $\|\boldsymbol{x}\|_\infty$ denotes the $L_2$ and $L_\infty$ norm of $\boldsymbol{x}$ respectively, and $\boldsymbol{x}^T$ is the transpose. The representations are also valid for matrices $A\in \mathbb{R}^{m\times n} $. The $m$-dimensional vector of zeros and ones is given by $\boldsymbol{0}_m$ and $\boldsymbol{1}_m $ respectively. $I$ and $ \otimes $ denote the identity matrix and the Kronecker product respectively. 
    
  \subsection{Graph Theory}

   In this paper, the interaction of sensing and communication among $n$ UAVs is represented using an undirected graph $ \mathcal{G}_f=(\mathcal{V}_f,\mathcal{E}_f) $, where $ \mathcal{V}_f=\{1,2...n\} $ denote the node set, and $ \mathcal{E}_f\subseteq\mathcal{V}_f\times\mathcal{V}_f $ denote the edge set. Furthermore, set the target with $i=0$, and the topology between UAVs and target is modeled by graph $ \mathcal{G}=(\mathcal{V},\mathcal{E}) $, where $ \mathcal{V}=\{0\} \cup \mathcal{V}_f $, and $ \mathcal{E}\subseteq\mathcal{V}\times\mathcal{V} $. The neighbors of UAV  $i\in\mathcal{V}_f$ are denoted by $\mathcal{N}_i=\{j\in\mathcal{V}_f:(i,j)\in\mathcal{E}_f\}$. The UAVs can obtain information from their neighbors by communication or sensing. The weighted adjacency matrix of $ \mathcal{G}$ is denoted by $\mathcal{A} =[a_{ij}]\in\mathbb{R}^{(n+1)\times(n+1)}$. If and only if $(i,j)\in\mathcal{E}$, then $a_{ij} > 0$, and $a_{ij}=0$, otherwise. Give the corresponding Laplacian matrix of the $ \mathcal{G} $ as $$ \mathcal{L} = \left[ l_{ij} \right]_{(n+1)\times (n+1)} = \left[ \begin{array}{cc} 0& \boldsymbol{0}_{1\times n} \\  \mathcal{L}_t& \mathcal{L}_f \\ \end{array} \right]$$ where $l_{ii}=\sum_{j=0}^{n}a_{ij}$ and $l_{ij}=-a_{ij}$ for $i\ne j$. It can be seen that $\mathcal{L}_t \triangleq \left[ \begin{matrix} -a_{10}& \dots & -a_{n0} \\ \end{matrix} \right]^T$ describes whether the UAV can estimate the state of the target, and $\mathcal{L}_f$ describes the internal communication in the UAVs. Since $ \mathcal{G}_f$ is an undirected graph, the weighted adjacency matrix $\mathcal{A}_f$ associated with graph $ \mathcal{G}_f$ satisfies $a_{ij}=a_{ji}$ for all $i,j\in\mathcal{V}_f$. Next, we will give the basic assumption and discuss the eigenvalue of $ \mathcal{L} $.
  
      \begin{assumption} \label{assump:graph}
          The graph $ \mathcal{G}_f$ is a connected, undirected, uniformly weighted, and circulant graph. In addition, assume that $a_{i0}>0, \forall i \in \mathcal{V}_f$, which means UAVs can always estimate the state of the target through direct measurements or information exchange between neighbors. 
      \end{assumption}
  
      \begin{lemma}\label{lemma:eigenvalue_L}
          Given that Assumption \ref{assump:graph} holds, the eigenvalues of $ \mathcal{L} $ are represented by $\lambda_0 = 0 < \lambda_1 \leq \dots \lambda_n$, where $\lambda_i \in \mathbb{R}, \forall i\in \mathcal{V}_f $ are the eigenvalues of $ \mathcal{L}_f $.
      \end{lemma}

      \begin{IEEEproof}
          It is well known that $ \mathcal{L} $ has a eigenvalue $\lambda_0 = 0$ with corresponding right eigenvector $\boldsymbol{1}_{n+1}$ \cite{eichlerClosedformSolutionOptimal2014}. The eigenvalues of $\mathcal{L}_f$ are denoted as $\lambda_i, i=1,\dots,n$, then 
              $$det(\lambda I - \mathcal{L}) = \begin{vmatrix} \lambda & \boldsymbol{0}_{1\times n} \\ -\mathcal{L}_t & \lambda I - \mathcal{L}_f \end{vmatrix} = (-1)^{1+1} \lambda det(\lambda I - \mathcal{L}_f)$$ 
          which means the eigenvalues $\lambda_i$ of $\mathcal{L}_f$ are also the eigenvalues of $\mathcal{L}$. Since the graph $ \mathcal{G}_f$ is undirected and circulant, then $\mathcal{A}_f^T = \mathcal{A}_f$, and $ \mathcal{L}_f$ is a real symmetric matrix.
          For all nonzero vectors $\boldsymbol{x} = \left[ \begin{matrix} x_{1}& \dots & x_{n} \end{matrix} \right]^T \in \mathbb{R}^m$, 

          $$\boldsymbol{x}^T \mathcal{L}_f \boldsymbol{x} = \frac{1}{2} \sum_{i=1}^{n}\sum_{j=1}^{n} a_{ij} (x_i - x_j)^2 + \sum_{i=1}^{n} a_{i0} x_i^2 > 0$$
          Hence, $\mathcal{L}_f$ is positive definite. To sum up, Lemma \ref{lemma:eigenvalue_L} is proved.
      \end{IEEEproof}
   
  \subsection{System Dynamics and Measurements}
  Without loss of generality,  the moving target and UAVs are characterized by the following discrete-time double-integrator model in some common coordinate frame $ \mathcal{F}_w 
   $.
   \begin{equation}\label{eq:dynamic}
                  \begin{bmatrix}  \boldsymbol{p}_{i,k+1} \\ \boldsymbol{v}_{i,k+1} \end{bmatrix} = \begin{bmatrix} I & TI \\ \boldsymbol{0} & I \end{bmatrix} \begin{bmatrix}  \boldsymbol{p}_{i,k} \\ \boldsymbol{v}_{i,k} \end{bmatrix} \\  + \begin{bmatrix}  \boldsymbol{0} \\ TI \end{bmatrix} \boldsymbol{u}_{i,k} 
                  + \begin{bmatrix}  \boldsymbol{0} \\ TI \end{bmatrix} \boldsymbol{\omega}_{i,k} 
      \end{equation}
  for $ i\in \mathcal{V}$, where $\boldsymbol{p}_{i,k}, \boldsymbol{v}_{i,k} \in \mathbb{R}^2$ are the position and velocity of the target or UAVs at time $kT$ where $T$ denote as sampling interval. $ \boldsymbol{\omega}_i \backsim \mathcal{N}(0,Q_i) $ is a zero-mean Gaussian noise with covariance matrix $Q_i$. For $i=0$, $ \boldsymbol{u}_{0,k} $ is an unknown bounded input of the target with $\| \boldsymbol{u}_{0,k} \| \leq U_0 $. When $ \boldsymbol{u}_{0,k} \equiv 0 $, the target is considered moving at constant velocity in the plane. It is assumed that the maximum speed of the target is $ v_{max0} \geq \| \boldsymbol{v}_{0,k} \|$. For all $ i \in \mathcal{V}_f $, $ \boldsymbol{u}_{i,k} $ is the bounded control input for UAVs with $\| \boldsymbol{u}_{i,k} \| \leq U_i $. And the maximum speed of UAVs is denoted as $ v_{maxi} \geq \| \boldsymbol{v}_{i,k} \|$.

      \begin{remark}
          In this paper, a ground target is considered. It is physically realistic for a target to be constrained by its maximum speed, and its acceleration more accurately reflects its maneuverability, such as turning capabilities. It is assumed that the UAVs are operating at a constant altitude.  Therefore, their positions will be mapped onto the $x,y$-plane, where the target is located. In this paper, we will focus solely on discussing positions within the $x,y$-plane.
      \end{remark}
  
     \begin{figure}[tbp]
      \centering
      \includegraphics[width=0.9\linewidth]{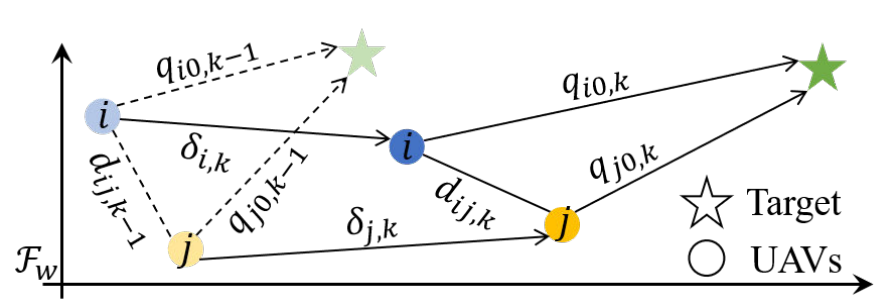}
      \caption{ Illustration of relative measurements between ``UAV-UAV" and ``UAV-Target" from $(k-1)T$ to $kT$.   }
      \label{fig:measurements}
    \end{figure}
  At each time step $k$, it is assumed that UAV $i$ can obtain its self-displacement measurement, denoted as $\boldsymbol{\delta}_{i,k}$, and transmit it to its neighbor UAV $j$ via the network. Thus, the relative displacement $\boldsymbol{\delta}_{ij,k}$ between UAV $i,j$ can be calculated. Additionally, the relative distance $d_{ij,k}$ between UAV $i,j$ can also be measured. Finally, the relative position measurement between UAV $i$ and the target, denoted as $ \boldsymbol{q}_{i0,k}$, is assumed to be obtained by onboard vision sensors. These measurements are illustrated in Fig.~\ref{fig:measurements}, and formally defined below. 
    \begin{equation}\label{eq:measurements}
    \begin{aligned}
            & \boldsymbol{\delta}_{i,k} \triangleq \boldsymbol{p}_{i,k} - \boldsymbol{p}_{i,k-1} + \boldsymbol{o}_{1,k} \\
            & \boldsymbol{\delta}_{ij,k}  \triangleq \boldsymbol{\delta}_{i,k} -\boldsymbol{\delta}_{j,k} \\
            & d_{ij,k} \triangleq \|\boldsymbol{p}_{ij,k} \| + \boldsymbol{o}_{2,k} \\
            & \boldsymbol{q}_{i0,k} \triangleq \boldsymbol{p}_{i0,k}  + \boldsymbol{o}_{3,k} , \quad i,j\in \mathcal{V}_f
          \end{aligned} 
    \end{equation}
  where $ \boldsymbol{o}_{m,k} \backsim \mathcal{N}(0,Q_m),m=1,2,3$ is the zero-mean Gaussian measurement noise with covariance matrix $Q_m$, and $\boldsymbol{p}_{ij,k} \triangleq \boldsymbol{p}_{i,k} - \boldsymbol{p}_{j,k}$ represents the relative position between UAVs or UAV and target.

      \begin{remark}
          The local frames of the UAVs are assumed to share a common orientation. This can be achieved by aligning all the frames with the Earth's magnetic frame or by consistently orienting the initial camera direction within the Visual-Inertial Odometry (VIO) system.
          It is reasonable to adapt different sensors to various subtasks in target enclosing and tracking. By using VIO or optical flow methods, self-displacement of UAVs can be directly obtained. In addition, distance measurement and communication can be achieved simultaneously via ultrawideband (UWB) sensors \cite{chungSurveyAerialSwarm2018,nguyenPersistentlyExcitedAdaptive2020,guoUltraWidebandOdometryBasedCooperative2020}. Vision-based relative position estimation between UAV and the target is widely studied \cite{liVGSwarmVisionBasedGene2023,tangOnboardDetectionTrackingLocalization2020}.
      \end{remark}

  \section{STATE ESTIMATION} \label{stateEsti}
  In this section, the relative state estimation framework is proposed (Fig.~\ref{fig:estimator}). The relative positions between UAVs are estimated with RLSE at first. Then, an improved DKF is applied by merging the output of the RLSE to achieve persistent cooperative state estimation for the target.

  \subsection{Relative Localization Between UAVs}
  As inspired by \cite{nguyenPersistentlyExcitedAdaptive2020}, the relative displacement and relative distance between UAVs form 
      $$  y_{ij,k} = \boldsymbol{\delta}_{ij,k}^{T}\boldsymbol{p}_{ij,k-1} $$
  where $y_{ij,k} \triangleq \frac{1}{2} \left[ d_{ij,k}^{2}- d_{ij,k-1}^{2}-\left\lVert \boldsymbol{\delta}_{ij,k} \right\rVert ^{2}  \right]$. Therefore, the problem of relative position estimation can be extended to the problem of estimating time-varying parameters. To solve the problem, it is well known by applying the RLSE technique as follows \cite{nguyenPersistentlyExcitedAdaptive2020}.
      \begin{equation}\label{eq:estimator}
              \begin{aligned}
                  & \boldsymbol{\bar{p}}_{ij,k} = \boldsymbol{\hat{p}}_{ij,k-1} + \boldsymbol{\delta}_{ij,k} \\
                  & \varGamma_{ij,k} = \frac{1}{\beta}\left[\varGamma_{ij,k-1} - \frac{\varGamma_{ij,k-1}\boldsymbol{\delta}_{ij,k} \boldsymbol{\delta}_{ij,k}^{T}\varGamma_{ij,k-1}}{\beta+\boldsymbol{\delta}_{ij,k}^{T}\varGamma_{ij,k-1}\boldsymbol{\delta}_{ij,k}}\right] \\
                  & \boldsymbol{\hat{p}}_{ij,k} = \boldsymbol{\bar{p}}_{ij,k} + \varGamma_{ij,k}\boldsymbol{\delta}_{ij,k} \left(y_{ij,k} - \boldsymbol{\delta}_{ij,k}^{T}\boldsymbol{\hat{p}}_{ij,k-1}\right)
              \end{aligned}
      \end{equation}	
  where $\boldsymbol{\hat{p}}_{ij,k}$ is the estimate of the relative position between UAV $i$ and $j$, $0<\beta<1$ is a forgetting factor, and $\varGamma_{ij,k} \in \mathbb{R}^{m\times m}$.
  
      \begin{remark}
      If the relative displacement $\boldsymbol{\delta}_{ij,k}$ satisfies the persistent excitation condition and without considering the measurement noise, the estimation error $ \boldsymbol{\tilde{p}}_{ij,k} \triangleq \boldsymbol{\hat{p}}_{ij,k} - \boldsymbol{p}_{ij,k}$ converges to $0$ exponentially fast \cite{nguyenPersistentlyExcitedAdaptive2020}. In this paper, the measurement noise is considered, and it can be predicted that the estimation error will converge to a bounded value. The influence of measurement noise on formation control is discussed in Section~\ref{Analysis}.
      \end{remark}

     \begin{figure}[tbp]
      \centering
      \includegraphics[width=\linewidth]{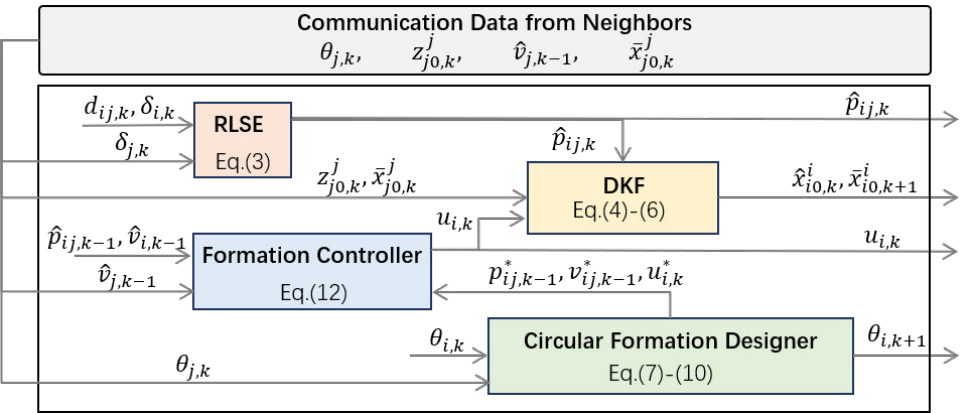}
      \caption{ The overall scheme for the target enclosing and tracking task. The RLSE and DKF methods implement the relative state estimation framework. The circular formation designer creates the desired formation pattern, while the formation controller manages the tracking control.}
      \label{fig:estimator}
    \end{figure}
  
  \subsection{Cooperated Target Motion Estimation}
  By making a difference to the dynamic model of UAV $i$ and target $0$ with $ \boldsymbol{u}_{0,k} \equiv 0 $, the process model is given as  
      $$ \boldsymbol{x}_{i0,k}^{i} = A \boldsymbol{x}_{i0,k-1}^{i} + B \boldsymbol{u}_{i,k-1} + G \boldsymbol{\omega}_{i0,k-1} $$
  where
      $$ \boldsymbol{x}_{i0,k}^{i} \triangleq \begin{bmatrix} \boldsymbol{p}_{i0,k} \\ \boldsymbol{v}_{0,k}^{i} \\  \boldsymbol{v}_{i,k} \end{bmatrix},  A \triangleq \begin{bmatrix} I & -TI & TI \\ \boldsymbol{0} & I & \boldsymbol{0} \\  \boldsymbol{0} & \boldsymbol{0} & I \end{bmatrix}, B = G\triangleq \begin{bmatrix} \boldsymbol{0} \\ \boldsymbol{0} \\  TI \end{bmatrix} $$
  and $ \boldsymbol{\omega}_{i0,k} \backsim \mathcal{N}(0,Q_{i0}) $ is a process noise satisfying a zero-mean Gaussian distribution with covariance matrix $Q_{i0}$. From \eqref{eq:measurements}, the measurement model for UAV $i$ is
      $$ \boldsymbol{z}_{i0,k}^{i} = H \boldsymbol{x}_{i0,k}^{i
      } + \boldsymbol{o}_{i0,k} $$
  where 
      $$\boldsymbol{z}_{i0,k}^{i} \triangleq \begin{bmatrix} \boldsymbol{q}_{i0,k} \\ \boldsymbol{\delta}_{i,k} \end{bmatrix}, H \triangleq \begin{bmatrix} I & \boldsymbol{0} & \boldsymbol{0} \\ \boldsymbol{0} & \boldsymbol{0} & TI \end{bmatrix} $$
  and $ \boldsymbol{o}_{i0,k} \backsim \mathcal{N}(0,R_{i0}) $ is the zero-mean Gaussian noise with covariance matrix $R_{i0}$. The superscript $i$ in $\boldsymbol{x}_{i0,k}^{i},  \boldsymbol{v}_{0,k}^{i}$ and $\boldsymbol{z}_{i0,k}^{i} $ indicates that these are direct estimates or measurements from UAV $i$.

    Then, the prediction step of DKF is given as
        \begin{equation}\label{eq:predictDKF}
              \begin{aligned}
                  & \boldsymbol{\bar{x}}_{i0,k}^{i} = A \boldsymbol{\hat{x}}_{i0,k-1}^{i} + B \boldsymbol{u}_{i,k-1} \\
                  & P_{i,k}^- = A P_{i,k-1}^+ A^T + G Q_{i0} G^T
              \end{aligned}
      \end{equation}
    where $ \boldsymbol{\bar{x}}_{i0,k}^{i} \triangleq \left[ \begin{matrix} (\boldsymbol{\bar{p}}_{i0,k})^T & (\boldsymbol{\bar{v}}_{0,k}^{i})^T & (\boldsymbol{\bar{v}}_{i,k})^T \end{matrix} \right]^T$ and $ \boldsymbol{\hat{x}}_{i0,k}^{i} \triangleq \left[ \begin{matrix} (\boldsymbol{\hat{p}}_{i0,k})^T & (\boldsymbol{\hat{v}}_{0,k}^{i})^T & (\boldsymbol{\hat{v}}_{i,k})^T \end{matrix} \right]^T$ are the priori and posteriori state estimate, respectively; and $ P_{i,k}^-$ and $ P_{i,k}^+$ are the priori and posteriori estimate error covariance, respectively. Typically, the measurements $\boldsymbol{z}_{j0,k}^{j}$ and estimate state $\boldsymbol{\bar{x}}_{j0,k}^{j}$ (or $\boldsymbol{\hat{x}}_{j0,k}^{j}$) from neighbor $j$ will be sent to UAV $i$ and be used directly in the DKF algorithms \cite{olfati-saberDistributedKalmanFiltering2007, liDistributedKalmanFilter2020, wangCooperativeTargetTracking2012}. However, due to the lack of global information, $\boldsymbol{\bar{x}}_{j0,k}^{j}$ and $\boldsymbol{z}_{j0,k}^{j}$ contain relative information with respect to the target in UAV $j$'s local frame. Therefore, this information must be aligned to UAV $i$'s local frame before it can be effectively used. Hence, the output of the RLSE \eqref{eq:estimator} is utilized to generate indirect estimates $\boldsymbol{\bar{x}}_{i0,k}^{j}$ and measurements $\boldsymbol{z}_{i0,k}^{j}$ for UAV $i$ from its neighbor $j$. 
    \begin{equation}\label{eq:z_x_change}
      \boldsymbol{\bar{x}}_{i0,k}^{j} \triangleq \begin{bmatrix} \boldsymbol{\bar{p}}_{j0,k} + \boldsymbol{\hat{p}}_{ij,k} \\ \boldsymbol{\bar{v}}_{0,k}^{j} \\  \boldsymbol{\bar{v}}_{i,k} \end{bmatrix}, \boldsymbol{z}_{i0,k}^{j} \triangleq \begin{bmatrix} \boldsymbol{q}_{j0,k} + \boldsymbol{\hat{p}}_{ij,k} \\ \boldsymbol{\delta}_{i,k} \end{bmatrix}     
     \end{equation}

  Following the correction step in \cite{olfati-saberDistributedKalmanFiltering2007}, there is
      \begin{equation}\label{eq:correctDKF}
              \begin{aligned}
                  \left(P_{i,k}^+ \right)^{-1} &  =  \left(P_{i,k}^-\right)^{-1} +  \sum_{j\in \mathcal{N}_i \cup {i} }  H^T R_{j0}^{-1} H  \\ 
                  \boldsymbol{\hat{x}}_{i0,k}^{i} & = \boldsymbol{\bar{x}}_{i0,k}^{i}  + \epsilon \sum_{j\in \mathcal{N}_i} \left( \boldsymbol{\bar{x}}_{i0,k}^{j} - \boldsymbol{\bar{x}}_{i0,k}^{i} \right) \\
                  & + P_{i,k}^+  \sum_{j\in \mathcal{N}_i \cup {i} } H^T R_{j0}^{-1} \left( \boldsymbol{z}_{i0,k}^{j} - H  \boldsymbol{\bar{x}}_{i0,k}^{i}   \right)
              \end{aligned}
      \end{equation}
  where $ \epsilon > 0 $ is a consensus gain. 
    
      \begin{remark}
       At the time step $k$, UAV $j$ send $\boldsymbol{z}_{j0,k}^{j}$, $\boldsymbol{\bar{x}}_{j0,k}^{j}$ and $\boldsymbol{\hat{v}}_{j,k-1}$ to all of its neighbors (Fig.~\ref{fig:estimator}). The vector $\boldsymbol{z}_{j0,k}^{j}$ contains $\boldsymbol{\delta}_{j,k}$, and $\boldsymbol{\hat{v}}_{j,k-1}$ is one of the components in the previous time estimate $\boldsymbol{\hat{x}}_{j0,k-1}^{j}$. The last term in $\boldsymbol{x}_{i0,k}^{i}$ and $\boldsymbol{z}_{i0,k}^{i}$ can be completely decoupled and is included solely to standardize the architecture. The correction step \eqref{eq:correctDKF} for $\boldsymbol{\hat{x}}_{i0,k}^{i}$ follow the classic paradigm outlined in \cite{olfati-saberDistributedKalmanFiltering2007, liDistributedKalmanFilter2020, wangCooperativeTargetTracking2012},  which is \emph{``posteriori = priori + innovation + consensus"}. Compared with the DKF algorithms proposed in \cite{olfati-saberDistributedKalmanFiltering2007, liDistributedKalmanFilter2020, wangCooperativeTargetTracking2012,lianDistributedKalmanConsensus2022,liSelfLocalizationDistributedSystems2022}, the need for global position information is avoided in this paper by integrating the relative position estimation between UAVs. 
       

      \end{remark}

  
  \section{CIRCULAR FORMATION PATTERN DESIGN}\label{formationDesign}
  In this section, a coupled oscillator-based method for designing the desired circular formation is proposed first. Then, the persistent excitation property of the desired relative velocity is briefly analyzed. 
  
  \subsection{Formation Pattern Planning}
  In this paper, UAVs are expected to circle the target in an evenly distributed pattern. Thus, the desired circular formation is always relative to the target's position $ \boldsymbol{p}_{0,k} $ as the center, with a specified distance $ \rho \in \mathbb{R} $. Assume that the desired circular speed of the UAVs relative to the target is $ v_c \in \mathbb{R}$. Then, there will produce a shift in phase $\Delta \theta$ at each sampling interval with the relation $$ 2 \rho \sin \frac{\Delta \theta}{2} = T v_c$$ This implies that a larger desired radius increases the time it takes for the UAV to orbit the target. To make sure the UAV can track the target, it is assumed that $ v_{maxi} > v_{max0} + v_c $.

  Let $\theta_{i,k} \in [0,2\pi)$ denote the angle of UAV $i$ relative to the $x$-axis in the target's local frame. Then, the desired position of UAV $i$ relative to the target can be uniquely determined as
      \begin{equation}\label{eq:desired_pos}
          \boldsymbol{p}_{i,k}^* = \rho  \left[ \begin{array}{cc} \cos{\theta_{i,k}} \\  \sin{\theta_{i,k}} \\ \end{array} \right]
      \end{equation}
  Further, the discrete-time desired velocity of UAV $i$ is
      \begin{equation}\label{eq:desired_vel}
          \boldsymbol{v}_{i,k}^* = \frac{1}{T} (\boldsymbol{p}_{i,k+1}^* - \boldsymbol{p}_{i,k}^*)
      \end{equation}
  Thus, the desired relative position and the desired relative velocity are denoted by $\boldsymbol{p}_{ij,k}^* \triangleq \boldsymbol{p}_{i,k}^* - \boldsymbol{p}_{j,k}^* $ and $\boldsymbol{v}_{ij,k}^* \triangleq \boldsymbol{v}_{i,k}^* - \boldsymbol{v}_{j,k}^* $, respectively. Specially, $\boldsymbol{p}_{i0,k}^* = \boldsymbol{p}_{i,k}^*$ and 
 $\boldsymbol{v}_{i0,k}^* = \boldsymbol{v}_{i,k}^*$. The desired acceleration is denoted by 
      \begin{equation}\label{eq:desired_acc} \boldsymbol{u}_{i,k+1}^* = \frac{1}{T} (\boldsymbol{v}_{i,k+1}^* - \boldsymbol{v}_{i,k}^*) = \frac{1}{T} (R_{\Delta \theta} - I ) \boldsymbol{v}_{i,k}^*\end{equation} 
  where $ R_{\Delta \theta} \triangleq \begin{bmatrix} \cos \Delta \theta & - \sin \Delta \theta \\ \sin \Delta \theta & \cos \Delta \theta \end{bmatrix} $. The dynamics of $\theta_{i,k}$ will be given by the coupled oscillator model \cite{liuFormationControlMoving2023,dorflerSynchronizationComplexNetworks2014}
      \begin{equation}\label{eq:oscillator}
          \begin{split}
              \theta_{i,k+1} &= \theta_{i,k} + \Delta \theta \\ &+ T\sum_{j = 1}^{n} \sum_{l = 1}^{n}  \frac{K_{l}a_{ij}}{l}\sin{\left(l\left[\theta_{i,k}-\theta_{j,k}\right]\right)}
          \end{split}
      \end{equation}
  where $K_{l}\in\mathbb{R} $ denotes the gains.

         \begin{figure}[tbp]
      \centering
      \includegraphics[width=0.9\linewidth]{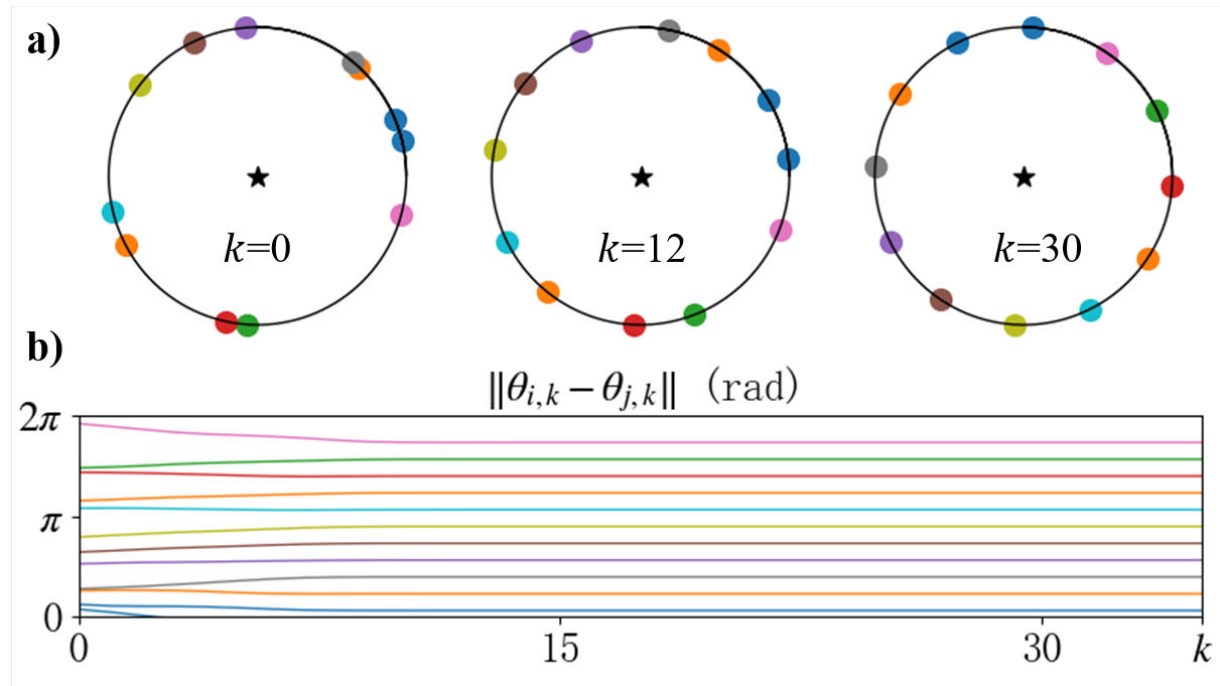}
      \caption{ Illustration of the circular formation pattern design model. a) A demonstration with 12 UAVs (circles) achieves the desired circular formation around the target (star) using the coupled oscillator model. b) The relative phase converges to a fixed value.}
      \label{fig:osc}
    \end{figure}
      \begin{remark}
      By choosing appropriate parameters, the desired circular formation can be achieved \cite{liuFormationControlMoving2023,dorflerSynchronizationComplexNetworks2014}. When the system reaches equilibrium, the relative phase between each drone will remain at a fixed value, as shown in Fig.~\ref{fig:osc}. Formally, there is $\lim_{k\rightarrow \infty} \| \theta_{i,k}-\theta_{j,k} \| = \varphi_{ij} $. Since the UAVs are evenly distributed around the circumference, $ \varphi_{ij} = 2\pi l /n, l=\pm 1,...,\pm n-1 $.
      
      The coupled oscillator-based formation pattern has the following advantages. Firstly, it allows the UAVs to form desired formations autonomously without the need to specify configurations in advance. Secondly, it can respond flexibly to conditions of reducing and increasing of UAVs as shown in our previous work \cite{liuFormationControlMoving2023}. Most importantly, the persistent excitation nature of the circular pattern will sufficiently guarantee the convergence of the RLSE, which will be discussed below. In addition, the formation can be changed counterclockwise or clockwise by changing the plus sign to the minus sign of $\Delta\theta$ in (\ref{eq:oscillator}).
      
      \end{remark}
  
         \begin{figure}[tbp]
      \centering
      \includegraphics[width=\linewidth]{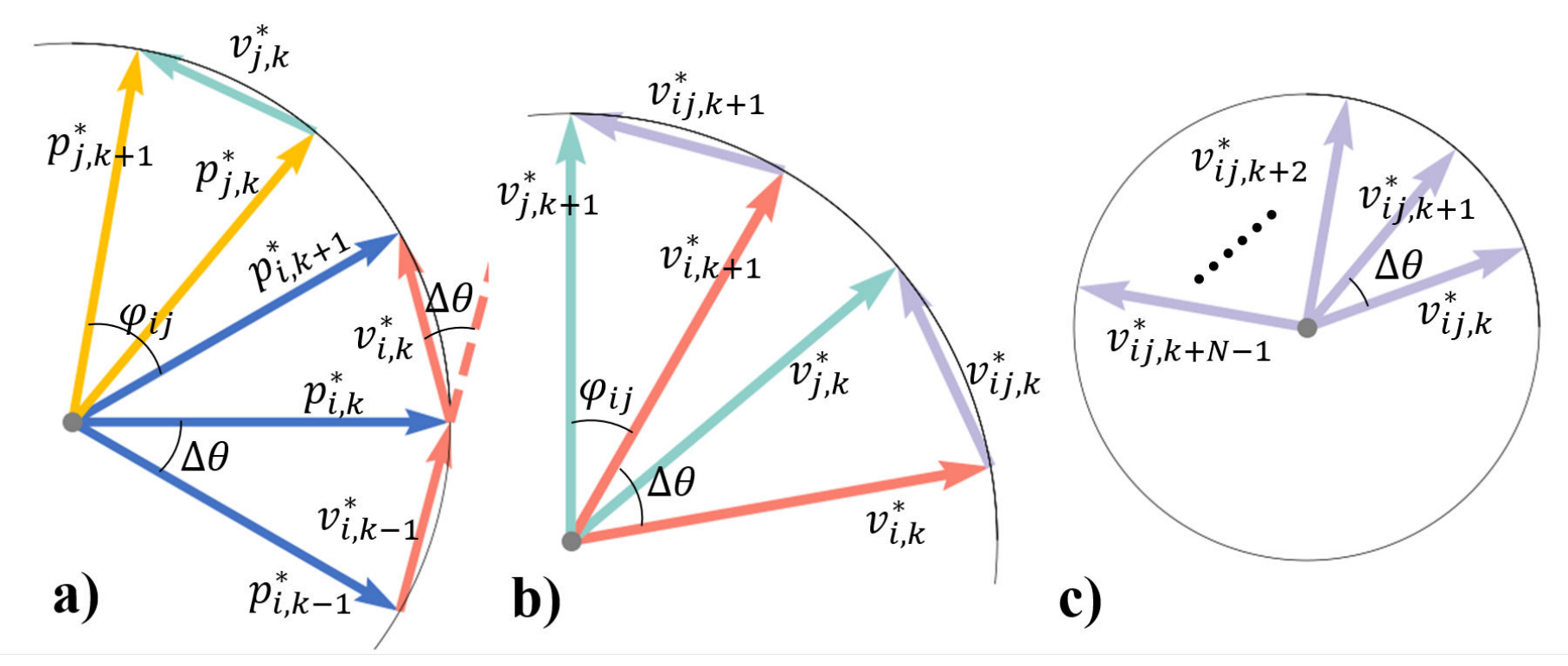}
      \caption{Illustration of the persistent excitation analysis for $\boldsymbol{v}_{ij}^*$. All vectors are mapped on the unit circle. a) The relation between $ \boldsymbol{p}_i^* $, $ \boldsymbol{p}_j^* $, $ \boldsymbol{v}_i^* $ and $ \boldsymbol{v}_j^* $. b) The relation between $ \boldsymbol{v}_i^* $, $ \boldsymbol{v}_j^* $ and $ \boldsymbol{v}_{ij}^*$. c) Over $N$ time steps, $ \boldsymbol{v}_{ij}^*$ are shown be non-collinear.} 
      \label{fig:PE}
    \end{figure}
  
  \subsection{Persistent Excitation Analysis}
  Intuitive analysis is given through geometric relations (Fig.~\ref{fig:PE}). Without prejudice to the conclusion, it is assumed that all vector changes are counterclockwise on a unit circle. It can be predicted that after a limited time step $k \geq K \in \mathbb{N}$, the coupled oscillator-based system (\ref{eq:oscillator}) will reach balance (Fig.~\ref{fig:osc}). This is mainly reflected in two aspects. First, $  \boldsymbol{p}_i^* $ changes $ \Delta \theta$ phase at each time step ($k$ is omitted without affecting understanding). Second, the relative phase between $  \boldsymbol{p}_j^* $ and $  \boldsymbol{p}_j^* $ will keep a constant value $ \varphi_{ij} \ne 0$. Thus, $ \boldsymbol{v}_i^* $ also changes $ \Delta \theta$ phase at each time step and keeps a constant phase $ \varphi_{ij}$ with $ \boldsymbol{v}_j^* $.   
  
  Hence, the desired relative velocity $ \boldsymbol{v}_{ij}^* $ also shifts the phase of $ \Delta \theta$ on the unit circle at each time step. If $ \Delta \theta \leq 2\pi/3 $, it is definitely to find $ \boldsymbol{v}_{ij,l_1}^* $ and $ \boldsymbol{v}_{ij,l_2}^* $ that's not collinear, where $ K\leq l_1<l_2 \in \mathbb{N}$. In fact, considering the performance constraints of UAVs, it is usually $ \Delta \theta \ll 2\pi/3 $. Formally, for any $\left(i,j\right)\in\mathcal{E}_f$, there exist $\alpha_{ij,2}\geq\alpha_{ij,1}>0$ and $K,N\in \mathbb{N} $ such that $\forall l\geq K$, $\boldsymbol{v}_{ij,k}^*$ satisfy the persistent excitation condition as follows
      \begin{equation}\label{eq:PE_v}
        \alpha_{ij,1}I \leqslant  \sum_{k=l}^{l+N-1}{\boldsymbol{v}_{ij,k}^* {\boldsymbol{v}_{ij,k}^*}^{T} }\leqslant \alpha_{ij,2}I
      \end{equation}

  
  \section{TARGET ENCLOSING AND TRACKING CONTROL}\label{formationControl}
  The distributed formation controller for simultaneously enclosing and tracking a moving target is proposed as follows.
      \begin{equation}\label{eq:formationtrackingcontrol}
          \left\{
               \begin{aligned}
                &\boldsymbol{u}_{i,k}^{trac} \triangleq -k_p \sum_{j\in\mathcal{N}_i\cup 0} a_{ij}\left({\hat{\boldsymbol{p}}}_{ij,k-1}-\boldsymbol{p}_{ij,k-1}^*\right) \\
                & \quad -k_v \sum_{j\in\mathcal{N}_i\cup 0} a_{ij}\left( {\hat{\boldsymbol{v}}}_{i,k-1} - {\hat{\boldsymbol{v}}}_{j,k-1} - \boldsymbol{v}_{ij,k-1}^* \right)  \\
                &\boldsymbol{u}_{i,k} = \pi_{U_{trac}}\left( \boldsymbol{u}_{i,k}^{trac} \right) + \boldsymbol{u}_{i,k}^*, i\in\mathcal{V}_f
                \end{aligned}
          \right.
      \end{equation}
  where $ \pi_{U}\left( \boldsymbol{u} \right) \triangleq \left(\min \{U,\| \boldsymbol{u} \|\} / \| \boldsymbol{u} \| \right) \boldsymbol{u}$ project $\boldsymbol{u}$ into bounded region, $U_{trac}>0$ is the upper bound, and $k_p,k_v>0$ are positive parameters. 
  
  \begin{remark} As shown in Fig.~\ref{fig:estimator}, at time step $k$, $\theta_{j,k}$ will be sent to UAV $i$ to calculate $\boldsymbol{p}_{ij,k-1}^*$ and $\boldsymbol{v}_{ij,k-1}^*$. Meanwhile, ${\hat{\boldsymbol{v}}}_{j,k-1}$, an element of $ \boldsymbol{\hat{x}}_{j0,k-1}^{j} $, is transmitted to UAV $i$ during the execution of the DKF algorithm.
  Limiting the upper bound avoids excessive control input when the estimation error is large, ensuring safety during the task.
  \end{remark}

  \section{THEORETICAL ANALYSIS}\label{Analysis}
  In this section, it is firstly shown that RLSE converges under the controller (\ref{eq:formationtrackingcontrol}). Then, performance analysis for a target moving at constant velocity and varying velocity are given respectively.
  Before the analysis is presented, some useful lemmas are given at first. 
  
      \begin{lemma}[\cite{gantmakher2000theory}]\label{lemma:det}
          Let $ A,B,C,D \in \mathbb{R}^{n\times n}, M = \begin{bmatrix} A & B \\ C & D\end{bmatrix}$. Then $ \det(M)=\det(AD-BC)$, if $ A,B,C$ and $D$ commute pairwise.
      \end{lemma}

      \begin{lemma}[\cite{bishop2011modern}]\label{lemma:routh}
          Let the characteristic polynomial of a second-order system is
              $$ q(s) = a_2 s^2 + a_1 s + a_0 $$ 
          The requirement for a stable second-order system is simply that all the coefficients be positive or all the coefficients be negative.
      \end{lemma} 
  
      \begin{lemma}[\cite{liesen2015linear}]\label{lemma:Kronecker}
          Let $ A,B \in \mathbb{R}^{n\times n}$, and suppose $ A \boldsymbol{u} = \lambda \boldsymbol{u}, B \boldsymbol{v} = \mu \boldsymbol{v} $. Then $ (A \otimes B)( \boldsymbol{u} \otimes \boldsymbol{v}) = (\lambda \mu)(\boldsymbol{u} \otimes \boldsymbol{v}) $, where $ \lambda$ and $\mu$ denote the eigenvalues of $A$ and $B$, and $\boldsymbol{u}, \boldsymbol{v}$ are the eigenvectors of $A$ and $B$, respectively. 
      \end{lemma} 
  
  \subsection{Convergence Analysis of RLSE}
      As mentioned in \cite{nguyenPersistentlyExcitedAdaptive2020}, to ensure the convergence of RLSE, $ \boldsymbol{\delta}_{ij,k} $ should be persistently exciting. 
      \begin{theorem}\label{them:rlsel}
          Under \eqref{eq:PE_v} and \eqref{eq:formationtrackingcontrol}, for any $\left(i,j\right)\in\mathcal{E}_f$, there exist $\gamma_{ij,2}\geq\gamma_{ij,1}>0$ and $K,N\in \mathbb{N}$, for all $l\geq K$ such that $\boldsymbol{\delta}_{ij,k}$ is strongly persistently exciting,
          \begin{equation}\label{eq:PE}
              \gamma_{ij,1}I \leqslant \boldsymbol{\Phi}_{ij,l}\triangleq\sum_{k=l}^{l+N-1}{\boldsymbol{\delta}_{ij,k}\boldsymbol{\delta}_{ij,k}^{T}}\leqslant \gamma_{ij,2}I
          \end{equation}	
      \end{theorem}
  
      \begin{proof}
          From \eqref{eq:dynamic}, \eqref{eq:desired_acc} and \eqref{eq:formationtrackingcontrol}, it can be obtained that $$ \delta_{ij,k+1} = \boldsymbol{c}_{ij,k-1}+TR_I \boldsymbol{v}_{ij,k-1}^* $$  where 
          $  \boldsymbol{c}_{ij,k-1} = T \boldsymbol{v}_{ij,k-1} + T^2 [ \pi_{U_{trac}}\left( \boldsymbol{u}_{i,k+1}^{trac} \right) - \pi_{U_{trac}}\left( \boldsymbol{u}_{j,k+1}^{trac} \right) ] $ and $ R_I \triangleq R_{\Delta \theta} - I $. Since $ \Delta \theta$ is constant, $\| R_I \| \equiv D >0$. Due to $\boldsymbol{u}_{i,k}^{trac}$ is bounded, there is $ \|\boldsymbol{c}_{ij}\|\leq C$, where $ 0< C$ is a constant value. In addition, $\| \boldsymbol{v}_{ij}^* \| \leq 2 v_c $. Follow the similar lines as in \cite[Theorem 1]{liuFormationControlMoving2023}, it can find $\gamma_{ij,1}= (C-2TDv_c)^2$ and $\gamma_{ij,2}= (C+2TDv_c)^2$. Thus, the persistent excitation of $\boldsymbol{\delta}_{ij,k}$ in \eqref{eq:PE} can be obtained.

      \end{proof}

  \subsection{Target Tracking Analysis at a Constant Velocity}
    The target tracking error is defined as
      \begin{equation}\label{eq:enclosing_goal}
                   \boldsymbol{e}_{t,k} \triangleq  \begin{bmatrix}  \bar{\boldsymbol{p}}_{k} - \boldsymbol{p}_{0,k} \\ \bar{\boldsymbol{v}}_{k} - \boldsymbol{v}_{0,k}  \end{bmatrix} 
     \end{equation}
              where $$ \bar{\boldsymbol{p}}_{k} = \frac{1}{n}\sum_{i=1}^n \boldsymbol{p}_{i,k}, \quad  \bar{\boldsymbol{v}}_{k}= \frac{1}{n}\sum_{i=1}^n \boldsymbol{v}_{i,k}, \quad i\in\mathcal{V}_f $$
              represent the position and velocity of the geometric center of the formation, respectively.
      \begin{theorem}\label{them:constant_target}
          Suppose Assumption \ref{assump:graph} holds, and select 
              \begin{equation}\label{eq:kpkv}
                  k_p T < k_v < \frac{4}{\lambda_{n} T U_{trac}}
              \end{equation}        
          If the target is moving at a constant speed, \emph{i.e.}, $ \boldsymbol{u}_{0,k} \equiv 0 $, the enclosing tracking error $ \boldsymbol{e}_{t,k} $ converges to 0 under the tracking control law (\ref{eq:formationtrackingcontrol}) in the absence of noise.
      \end{theorem}

      \begin{proof}
          Without considering noise, it can be seen that $ \lim_{k\rightarrow\infty} \| \boldsymbol{\tilde{p}}_{ij,k} \| \triangleq \lim_{k\rightarrow\infty} \| \boldsymbol{\hat{p}}_{ij,k} - \boldsymbol{p}_{ij,k} \| =0 $ and $ \lim_{k\rightarrow\infty} \| \boldsymbol{\tilde{v}}_{ij,k} \| \triangleq \lim_{k\rightarrow\infty} \| \boldsymbol{\hat{v}}_{ij,k} - \boldsymbol{v}_{ij,k} \| =0 $. Let $ \boldsymbol{e}_{pi,k} = \boldsymbol{p}_{i0,k} - \boldsymbol{p}_{i0,k}^* $, $ \boldsymbol{e}_{vi,k} = \boldsymbol{v}_{i0,k} - \boldsymbol{v}_{i0,k}^*$, and denote 
          $$ \boldsymbol{e}_{k} \triangleq \left[ \begin{matrix} \boldsymbol{e}_{p1,k}^T & \dots & \boldsymbol{e}_{pn,k}^T & \boldsymbol{e}_{v1,k}^T & \dots & \boldsymbol{e}_{vn,k}^T \end{matrix} \right]^T $$ Thus, to show the main goal  $ \lim_{k\rightarrow\infty}  \| \boldsymbol{e}_{t,k} \|_{\infty} = 0$, it is equal to prove $ \lim_{k\rightarrow\infty}  \| \boldsymbol{e}_{k} \|_{\infty} = 0$.
          
          
          Rewritten $\boldsymbol{u}_{i,k+1}^{trac}$ in \eqref{eq:formationtrackingcontrol}  as
              $$ \begin{aligned}
                &\boldsymbol{u}_{i,k+1}^{trac} \triangleq -k_p \sum_{j\in\mathcal{N}_i} a_{ij}\left( {\boldsymbol{e}_{pi,k}}-\boldsymbol{e}_{pj,k} + \boldsymbol{\tilde{p}}_{ij,k} \right) \\
                & \qquad \qquad -k_v \sum_{j\in\mathcal{N}_i} a_{ij}\left( \boldsymbol{e}_{vi,k} - \boldsymbol{e}_{vj,k} + \boldsymbol{\tilde{v}}_{ij,k} \right)  \\
                \end{aligned}
              $$ 
          Combing the dynamic model \eqref{eq:dynamic}, desired formation pattern \eqref{eq:desired_pos}, \eqref{eq:desired_vel}, \eqref{eq:desired_acc} and applying (\ref{eq:formationtrackingcontrol}), there is
              \begin{equation*}
                  \begin{aligned}
                      & \begin{bmatrix} \boldsymbol{e}_{pi,k+1} \\ \boldsymbol{e}_{vi,k+1} \end{bmatrix} = \left( \begin{bmatrix} I & TI \\ \boldsymbol{0} & I \end{bmatrix} - s_{i,k} T l_{ii} \begin{bmatrix} \boldsymbol{0} & \boldsymbol{0} \\ k_p I & k_v I \end{bmatrix} \right) \begin{bmatrix}  \boldsymbol{e}_{pi,k} \\ \boldsymbol{e}_{vi,k} \end{bmatrix} \\ 
                      &\qquad + s_{i,k} T \sum_{j\in \mathcal{V}_f}  a_{ij} \begin{bmatrix} \boldsymbol{0} & \boldsymbol{0} \\ k_p I & k_v  I \end{bmatrix} \left( \begin{bmatrix}  \boldsymbol{e}_{pj,k} \\ \boldsymbol{e}_{vj,k} \end{bmatrix}-\begin{bmatrix}  \boldsymbol{\tilde{p}}_{ij,k} \\ \boldsymbol{\tilde{v}}_{ij,k} \end{bmatrix} \right) \\
                  \end{aligned}
              \end{equation*}
            where $l_{ii}=\sum_{j=0}^{n}a_{ij}$ in the graph $ \mathcal{G} $, and $$s_{i,k}= \frac{\min \{U_{trac},\| \boldsymbol{u}_{i,k+1}^{trac} \|\}}{\| \boldsymbol{u}_{i,k+1}^{trac} \|}  < U_{trac}$$

          For all UAV $ i\in \mathcal{V}_f $, the error dynamics is further rewritten as
              \begin{equation}\label{eq:error_dynamic}
                  \begin{aligned}
                      \boldsymbol{e}_{k+1} = E \otimes I \boldsymbol{e}_{k} - \boldsymbol{\tilde{e}}_{k}
                  \end{aligned}
              \end{equation}  
          where
              $$ E = \begin{bmatrix} I & TI \\ -k_p T s_{i,k} \mathcal{L}_f & I - k_v T s_{i,k} \mathcal{L}_f \end{bmatrix} $$
          and $ \boldsymbol{\tilde{e}}_{k} \triangleq \left[ \begin{matrix} \boldsymbol{0}_n^T & \boldsymbol{\tilde{e}}_{1,k}^T & \dots & \boldsymbol{\tilde{e}}_{n,k}^T \end{matrix} \right]^T$ with 
              $$ \boldsymbol{\tilde{e}}_{i,k} = s_{i,k} T \sum_{j\in \mathcal{V}_f} a_{ij} \begin{bmatrix} k_p I & k_v I \end{bmatrix} \begin{bmatrix}  \boldsymbol{\tilde{p}}_{ij,k} \\ \boldsymbol{\tilde{v}}_{ij,k} \end{bmatrix} $$
          Thus, $ \lim_{k\rightarrow\infty} \| \boldsymbol{\tilde{e}}_{k} \| =0 $.
          Here, the dimension of matrix $I$ is omitted which can be easily inferred.
          
          The eigenvalues of $E \otimes I$ mainly depend on the eigenvalues of $E$ (Lemma \ref{lemma:Kronecker}). Thus, the convergence analysis of the error system (\ref{eq:error_dynamic}) is mainly relative to the eigenvalues of $E$. By Lemma \ref{lemma:eigenvalue_L} and Lemma \ref{lemma:det}, the characteristic polynomial of $E$ is written as
              \begin{equation*}
                  \begin{aligned}
                      \det(s I_{2n} - E)
                      &= \prod_{i=1}^{n} \bigg( s^2 + (k_v T s_{i,K} \lambda_i - 2) s \\
                      &\qquad + k_p T^2 s_{i,K} \lambda_i - k_v T s_{i,K} \lambda_i + 1 \bigg)\\
                  \end{aligned}
              \end{equation*}        
          Thus, it is obvious that the eigenvalues of $E$ satisfy
              \begin{equation*}
                  \begin{aligned}
                      p(s,\lambda_i) &\triangleq s^2 + (k_v T s_{i,k} \lambda_i - 2) s \\
                      &\qquad \quad + k_p T^2 s_{i,k} \lambda_i - k_v T s_{i,k} \lambda_i + 1 = 0
                  \end{aligned}
              \end{equation*}
          After that, applying the bilinear transformation, which is defined by $ s = \frac{\tau+1}{\tau-1}$, and it maps the roots in the unit circle into the left complex plane \cite{eichlerClosedformSolutionOptimal2014}. Then, it can be obtained that
              \begin{equation*}
                  \begin{aligned}
                      p(\tau,\lambda_i) &= (k_p T^2 s_{i,k} \lambda_i)\tau^2  + 2(k_v T s_{i,k} \lambda_i-k_p T^2 s_{i,k} \lambda_i)\tau \\ & \qquad+ k_p T^2 s_{i,k} \lambda_i - 2 k_v T s_{i,k} \lambda_i + 4
                  \end{aligned}
              \end{equation*} 
          Based on Lemma \ref{lemma:routh}, if all the coefficients of $p(\tau,\lambda_i)$ are positive, then the error dynamics $ \boldsymbol{e}_{k+1} = E \otimes I \boldsymbol{e}_{k} $ will be stable. By chosen $ k_p>0 $, the first coefficient $ k_p T^2 s_{i,k} \lambda_i $ is always positive for all $i$ and $k$. Thus, by solving
              \begin{equation*}
                  \left\{
                       \begin{aligned}
                        &-k_p T^2 s_{i,k} \lambda_i + k_v T s_{i,k} \lambda_i > 0 \\
                        & k_p T^2 s_{i,k} \lambda_i - 2 k_v T s_{i,k} \lambda_i + 4 > 0 \\
                        \end{aligned}
                  \right.
              \end{equation*}        
         for $ \forall i \in  \mathcal{V}_f$, and $ \forall k \in  \mathbb{N}$, the condition \eqref{eq:kpkv} is satisfied.
          
          Based on the above analysis, there exists an orthogonal matrix $ \Omega $ that satisfies $E=\Omega \Sigma \Omega^{-1} $, where $ \Sigma = \text{diag}\left( \sigma_1, \sigma_2, \dots, \sigma_{2n}  \right) $ with $ \sigma_{i}, i=1,2,\dots,2n$ are the eigenvalues of $E$. In addition, with \eqref{eq:kpkv} hold, all the eigenvalues of $E$ are within the unit circle, \emph{i.e.}, $ |\sigma_i| < 1$, for all $ i=1,2,\dots,2n$. By applying the triangle inequality
              \begin{equation*}
                  \begin{aligned}
                      \| \boldsymbol{e}_{k+1} \|_{\infty} &\leq \|E\|_{\infty} \|\boldsymbol{e}_{k}\|_{\infty} + \|\boldsymbol{\tilde{e}}_{k}\|_{\infty}\\
                      & \leq \|\Omega\|_{\infty} \|\Omega^{-1}\|_{\infty}  \|\Sigma\|_{\infty}  \|\boldsymbol{e}_{k}\|_{\infty} + \|\boldsymbol{\tilde{e}}_{k}\|_{\infty}\\ 
                      & = \sigma_{max} \|\boldsymbol{e}_{k}\|_{\infty} + \|\boldsymbol{\tilde{e}}_{k}\|_{\infty}\\
                      & \leq \sigma_{max}^{k+1} \|\boldsymbol{e}_{0}\|_{\infty} + \sum_{i=0}^{k} \sigma_{max}^{k-i} \|\boldsymbol{\tilde{e}}_{i}\|_{\infty}
                  \end{aligned}
              \end{equation*}          
          where $ \sigma_{max} = \max_{i=1,2,\dots,2n}\{ |\sigma_i| \} \in (0,1) $. 
          Since $ \|\boldsymbol{e}_{0}\|_{\infty} $ is bounded for any initial state, there is
              $$ \lim_{k\rightarrow\infty} \sigma_{max}^{k+1} \|\boldsymbol{e}_{0}\|_{\infty} =0 $$
          Furthermore, $ \boldsymbol{\tilde{e}}_{k} $ is bounded and converge to $0$ as $k\rightarrow\infty$. In other words, for every $ \varepsilon > 0$, there exists a positive integer $ N $ such that $ \|\boldsymbol{\tilde{e}}_{k}\| < \varepsilon $ for all $ k \ge N$. Thus,
              \begin{equation*}
                  \begin{aligned}
                      \sum_{i=0}^{k} \sigma_{max}^{k-i} \|\boldsymbol{\tilde{e}}_{i}\|_{\infty} & = \sum_{i=0}^{N} \sigma_{max}^{k-i} \|\boldsymbol{\tilde{e}}_{i}\|_{\infty} + \sum_{i=N+1}^{k} \sigma_{max}^{k-i} \|\boldsymbol{\tilde{e}}_{i}\|_{\infty} \\
                      &< \sum_{i=0}^{N} \sigma_{max}^{k-i} \|\boldsymbol{\tilde{e}}_{i}\|_{\infty} + \frac{1-\sigma_{max}^{k-N-1}}{1-\sigma_{max}} \varepsilon
                  \end{aligned}
              \end{equation*}
          For the determinate $N$, as $ k\rightarrow\infty $, it can be obtained that
              \begin{equation*}
                  \begin{aligned}
                      \lim_{k\rightarrow\infty} \sum_{i=0}^{k} \sigma_{max}^{k-i} \|\boldsymbol{\tilde{e}}_{i}\|_{\infty} = 0
                  \end{aligned}
              \end{equation*}        
          To sum up, $ \lim_{k\rightarrow\infty}  \| \boldsymbol{e}_{k} \|_{\infty} = 0$ , \emph{i.e.}, the system (\ref{eq:error_dynamic}) is convergent. Consequently, the Theorem \ref{them:constant_target} is proved.      
      \end{proof}
  
  \subsection{Target Tracking Analysis at a Varying Velocity}
  In practice, measurement noise and disturbances are unavoidable. Therefore, the estimation error is expected to converge to $ \eta > 0$, where $ \eta $ is the upper bound of the error. Formally, there is $\lim_{k\rightarrow\infty} \| \boldsymbol{\tilde{e}}_{k} \| = \eta $. In addition, the state of the target motion may not be at a constant velocity whether due to process noise $\boldsymbol{\omega}_{0,k}$ or due to the changing input $ \boldsymbol{u}_{0,k} $. To this end, the condition of considering noise and bounded input for the target is invested in this paper, which is given below.
      \begin{theorem}\label{them:nonconstant_target}
          Suppose Assumption \ref{assump:graph} holds, and select $k_p,k_v$ by the condition (\ref{eq:kpkv}), if the target driven by an unknown bounded input $ \| \boldsymbol{u}_{0,k} \| < U_0 $, the enclosing tracking problem can be solved with the ultimate bound $ (\eta + T U_0)/(1-\sigma_{max}) $ under the tracking control law (\ref{eq:formationtrackingcontrol}).
      \end{theorem}
      
      \begin{proof}
          The error dynamics is given as
              \begin{equation}\label{eq:error_dynamic_2}
                  \begin{aligned}
                      \boldsymbol{e}_{k+1} = E \otimes I \boldsymbol{e}_{k} - \boldsymbol{\tilde{e}}_{k} - T \boldsymbol{u}_{n0,k+1}
                  \end{aligned}
              \end{equation}
          where $ \boldsymbol{u}_{n0,k} \triangleq \left[ \begin{matrix} \boldsymbol{0}_n^T & \boldsymbol{1}_n^T \otimes \boldsymbol{u}_{0,k}^T \end{matrix} \right]^T$ with $ \| \boldsymbol{u}_{n0,k} \|_{\infty} < U_0 $. 
           By applying the triangle inequality to (\ref{eq:error_dynamic_2})
              \begin{equation*}
                  \begin{aligned}
                      \| \boldsymbol{e}_{k+1} \|_{\infty} &\leq \sigma_{max} \|\boldsymbol{e}_{k}\|_{\infty} + \|\boldsymbol{\tilde{e}}_{k}\|_{\infty} + T \| \boldsymbol{u}_{n0,k} \|_{\infty}  \\
                      & \leq \sigma_{max}^{k+1} \|\boldsymbol{e}_{0}\|_{\infty} + \frac{1-\sigma_{max}^{k}}{1-\sigma_{max}} ( \eta + T U_0) \\
                  \end{aligned}
              \end{equation*}       
          Hence
          $$ \lim_{k\rightarrow\infty} \| \boldsymbol{e}_{k+1} \|_{\infty} \leq \frac{\eta + T U_0}{1-\sigma_{max}} $$
          Consequently, the Theorem \ref{them:nonconstant_target} is proved.
      \end{proof}
  
  \section{SIMULATIONS AND EXPERIMENTS}\label{sim&exp}
  Simulation and experiment results are given in this section. Three different scenarios of four UAVs enclosing and tracking a ground target are considered, \emph{e.g.} a constant velocity target, a varying velocity target, and a constant velocity target in an environment with obstacles. 
  
  The communication topology among the UAVs is  $$\mathcal{A}_f=a_{ij}\begin{bmatrix} 0 & 1 & 0 & 1 \\ 1 & 0 & 1 & 0 \\ 0 & 1 & 0 & 1 \\ 1 & 0 & 1 & 0 \end{bmatrix}$$ with $ a_{ij}=1/3$. The covariance of process noise of the target and UAVs are $ Q_0 = 0.04 I $ and $ Q_i = 0.001 I$ respectively. The covariance of the measurement noise is $ R_i = 10^{-4}I$. And set the sample time $T=0.1s$. For the relative localization algorithm (\ref{eq:estimator}), let $ \varGamma_{ij,0} = I $ and $\beta = 0.9$. In the DKF algorithm, let $ \epsilon = 0.1 $, $P_{i,0}^{+}=I$ and the initial state is given as $\boldsymbol{\bar{x}}_{i0,0}=\boldsymbol{0} $. The $ \theta_{i,0} $ for all UAVs is given randomly. Let $ \Delta \theta = 0.05 rad $, $K_l = 1, l=1,...,n-1$, and $K_n = -1$. For the controller (\ref{eq:formationtrackingcontrol}), let $k_p=0.9$, $k_v=0.5$ and $U_{trac}=0.7$.
    
  The trajectory and error analysis are given in each experiment, in which $\| \boldsymbol{e}_{t,k} \|$ is the formation tracking error and ${\max_{(i,j)\in\mathcal{E}_f}\| \boldsymbol{\tilde{p}}}_{ij,k}\|$ denotes the relative position estimation error. The DKF estimation errors are evaluated by ${\max_{i\in\mathcal{V}_f}\|\boldsymbol{\tilde{p}}}_{i0,k}\|$ and ${\| \boldsymbol{\tilde{v}}}_{0,k}\|$. $\max_{i\in\mathcal{V}_f}\|\boldsymbol{e}_{pi,k}\|$ denotes the formation error, that is, the ``distance'' to the desired uniform distribution on a circle.
  
  \subsection{Numerical Simulations}
   The simulation scenarios are constructed in AirSim\footnote{\url{https://microsoft.github.io/AirSim/}}, which is a high-fidelity and physical simulation platform for autonomous vehicles.  The initial positions of the UAV projection to the $x,y$ plane are $[1,-3]^T,[-3,1]^T,[-1,2]^T$ and $[-2,-1]^Tm$. All initial velocities of UAVs are $[0,0]^T m/s$. The initial position of the vehicle is $[0,0]^Tm$. In AirSim, a facing down camera with $90^\circ$ FoV is installed on each UAV. The UAVs get measurements relative to the vehicle by using the \emph{Object Detection} API provided by AirSim. The UAVs are expected to be evenly distributed in a circle with a radius of $\rho=4m$ centered on the location of the vehicle. At the start of the simulations, the UAVs hover at an altitude of $15m$ above the ground.
  
  1) \emph{Target at a Constant Velocity:} In the first simulation, the UAVs are expected to track the vehicle moving in constant velocity $ \boldsymbol{v}_{0} = [1,0]^T m/s $. The trajectories of the vehicle and UAVs are shown in Fig.~\ref{fig:constant_velocity_traj}. The formation center path gradually converges to the target, and the UAVs are evenly distributed around the desired circumference. The tracking and estimation errors are given in Fig.~\ref{fig:constant_velocity_error}. Due to sensor noise, errors do not converge to zero ideally but instead converge to a small value. 
    \begin{figure}[tbp]
      \centering
      \includegraphics[width=0.9\linewidth]{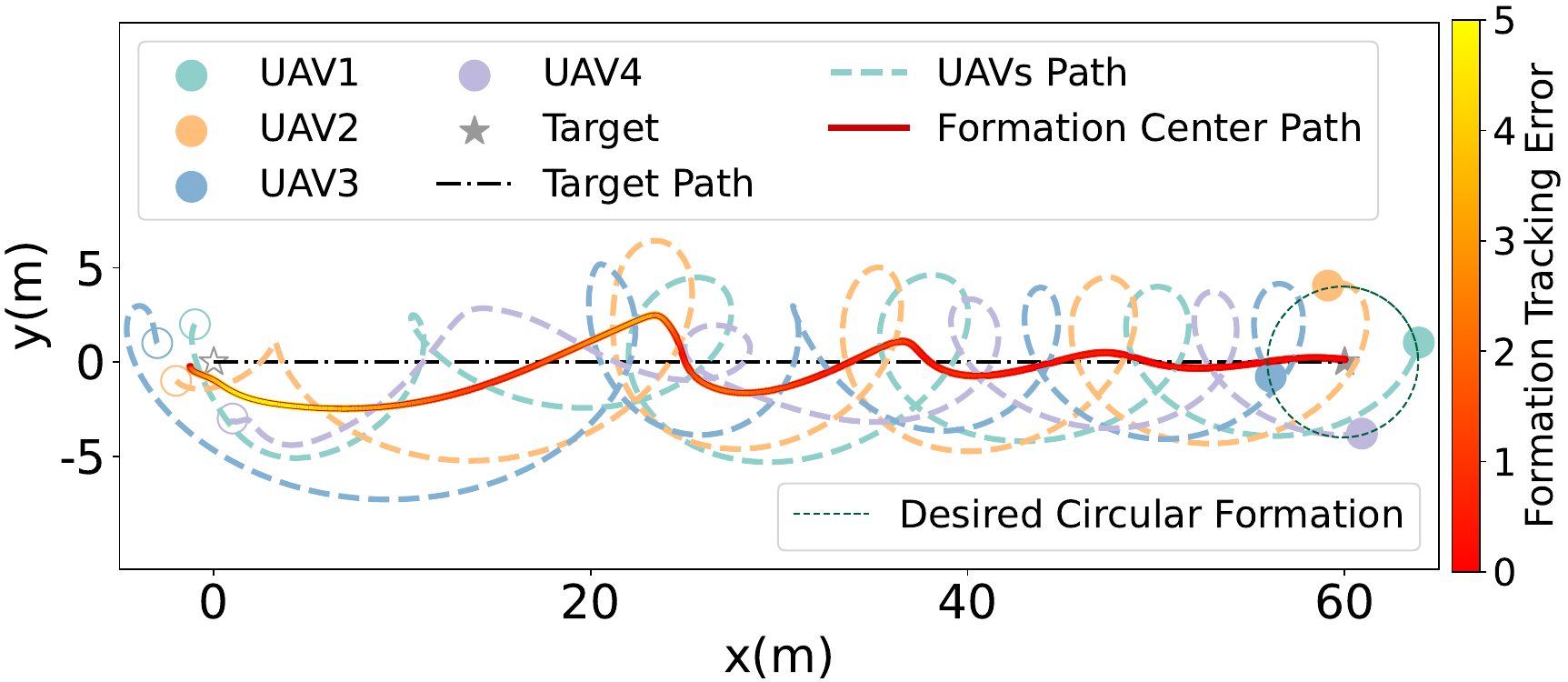}
      \caption{ The paths of four UAVs enclosing and tracking a target with a constant velocity in simulation. }
      \label{fig:constant_velocity_traj}
    \end{figure}
   
     \begin{figure}[tbp]
      \centering
      \includegraphics[width=\linewidth]{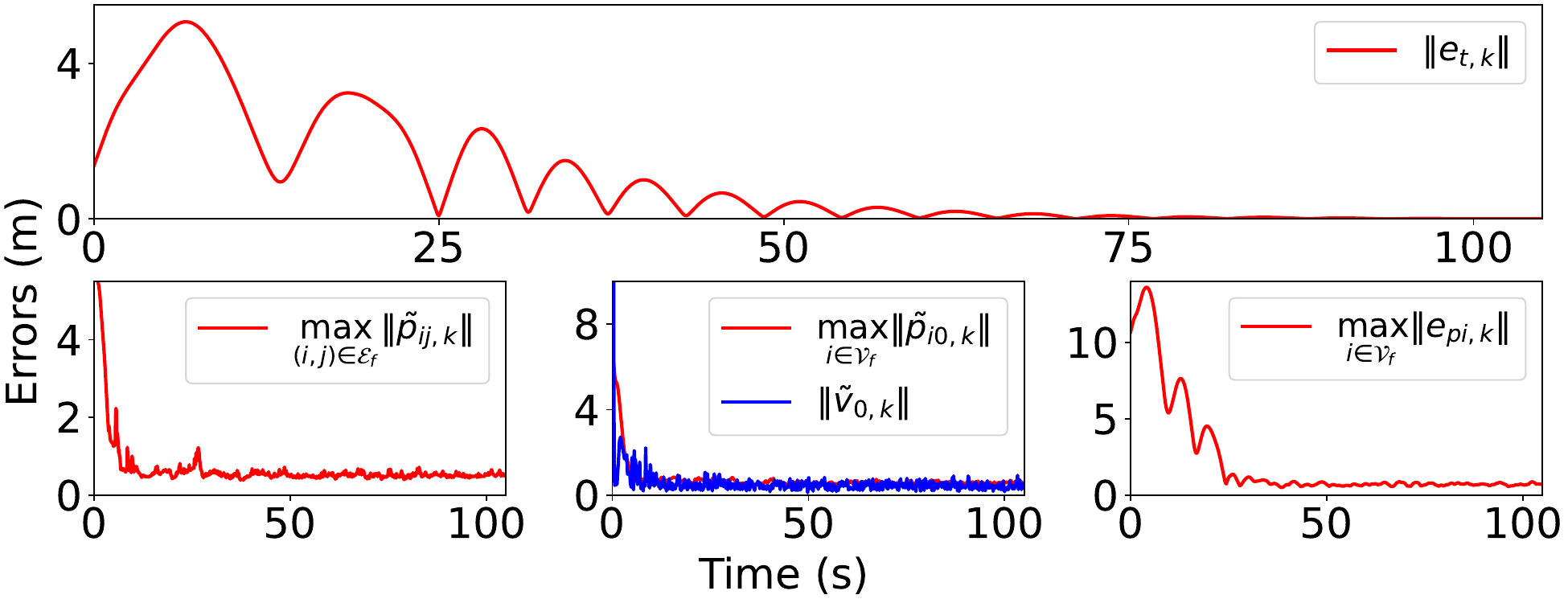}
      \caption{ Error analysis for the simulation of target tracking at a constant velocity.   }
      \label{fig:constant_velocity_error}
    \end{figure}

  2) \emph{Target at a Varying Velocity:} As mentioned in Theorem \ref{them:nonconstant_target}, the formation tracking algorithm proposed in the paper could also handle tracking a moving target with varying velocity. In this simulation, the vehicle is moving with a changing input $$  \boldsymbol{u}_{0,k} = \begin{bmatrix} \cos{\frac{2\pi k}{450}} & -\sin{\frac{2\pi k}{450}}\\ \sin{\frac{2\pi k}{450}} & \cos{\frac{2\pi k}{450}} \end{bmatrix} \begin{bmatrix}  0 \\ 0.1 \end{bmatrix} m/s^2 $$
  The paths of the target and UAVs are shown in Fig.~\ref{fig:acc_traj}. The formation center path gradually maintains a fixed error distance with the target path. The estimation and tracking errors are depicted in Fig.~\ref{fig:acc_error}.  
  
    \begin{figure}[tbp]
      \centering
      \includegraphics[width=0.9\linewidth]{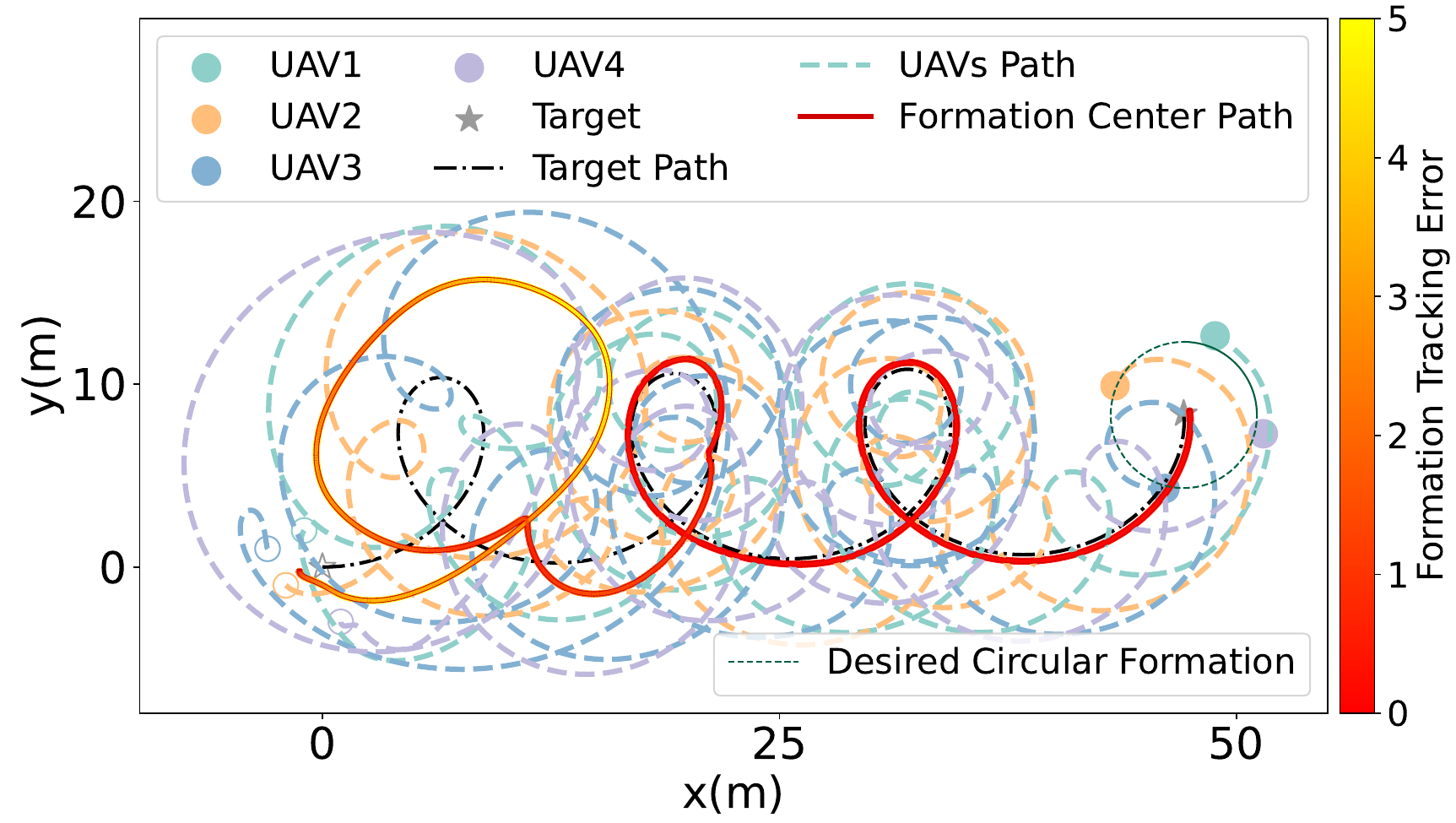}
      \caption{The paths of four UAVs enclosing and tracking a target at a varying velocity in simulation. }
      \label{fig:acc_traj}
    \end{figure}

     \begin{figure}[tbp]
      \centering
      \includegraphics[width=\linewidth]{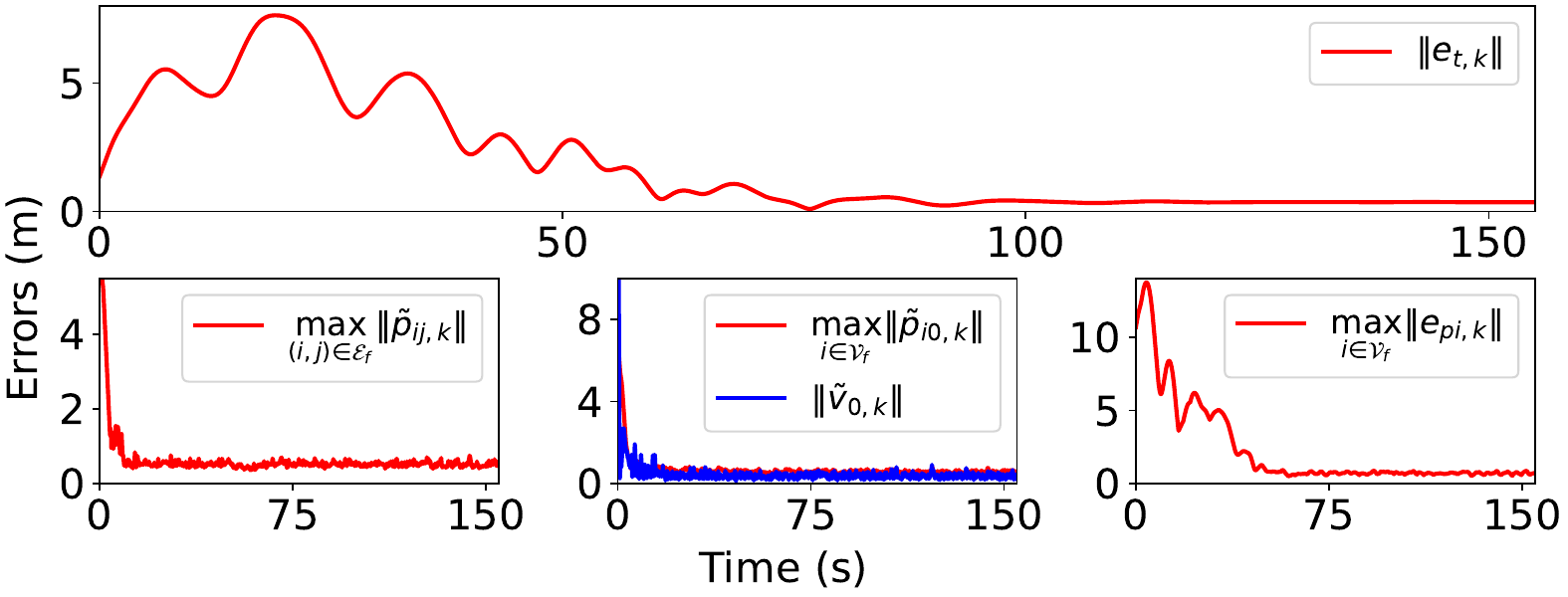}
      \caption{Error analysis for the simulation of target tracking at a varying velocity. }
      \label{fig:acc_error}
    \end{figure}

       \begin{figure}
      \centering	\includegraphics[width=0.85\linewidth]{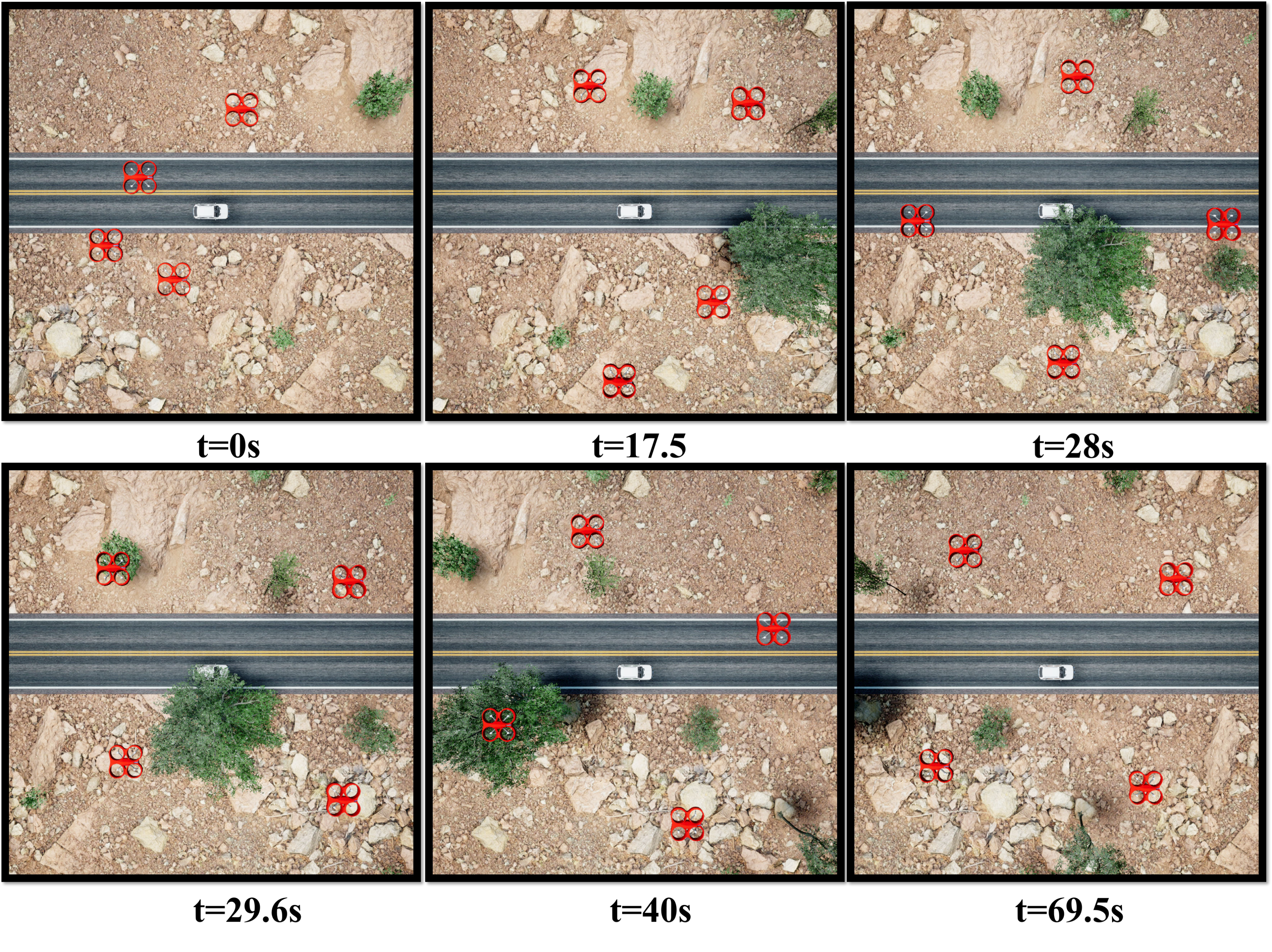}
      \caption{ Formation tracking process with vision occlusion in AirSim scenario. Trees were placed near the vehicle's path, thus obstructing the view from the UAVs. }
      \label{fig:time_airsim}
    \end{figure}
    
  3) \emph{Vision Occlusion:} In this simulation, a scenario with vision occlusion is constructed. Trees were placed near the path of the vehicle in AirSim, as shown in Fig.~\ref{fig:time_airsim}. The velocity of the vehicle is set to $ \boldsymbol{v}_{0} = [1,0]^T m/s $. During the formation tracking process, certain UAVs lost the sensing of the target from $27.3s$ to $34.2s$, due to the occlusion of trees. But other UAVs can still get measurements of the vehicle from other locations, as shown in Fig.~\ref{fig:tracking}. The state estimation of the target can be maintained continuously through information exchange and relative position measurement. The paths of the target and UAVs are shown in Fig.~\ref{fig:traj_block}, and areas with occlusions are filled in green color. The estimation and tracking errors are shown in Fig.~\ref{fig:errors_blocks}. Due to occlusion, the state estimation error in the DKF algorithm will be increased temporarily. This is mainly caused by applying the indirect measurement $ \boldsymbol{\hat{p}}_{ij} +  \boldsymbol{\hat{q}}_{j0}$ as input of DKF. However, thanks to the upper bound of controller \eqref{eq:formationtrackingcontrol}, the increase in estimation error does not greatly affect the tracking process. As the target reappears, the estimation error converges. 
  
  

     \begin{figure}
              \centering
      \includegraphics[width=\linewidth]{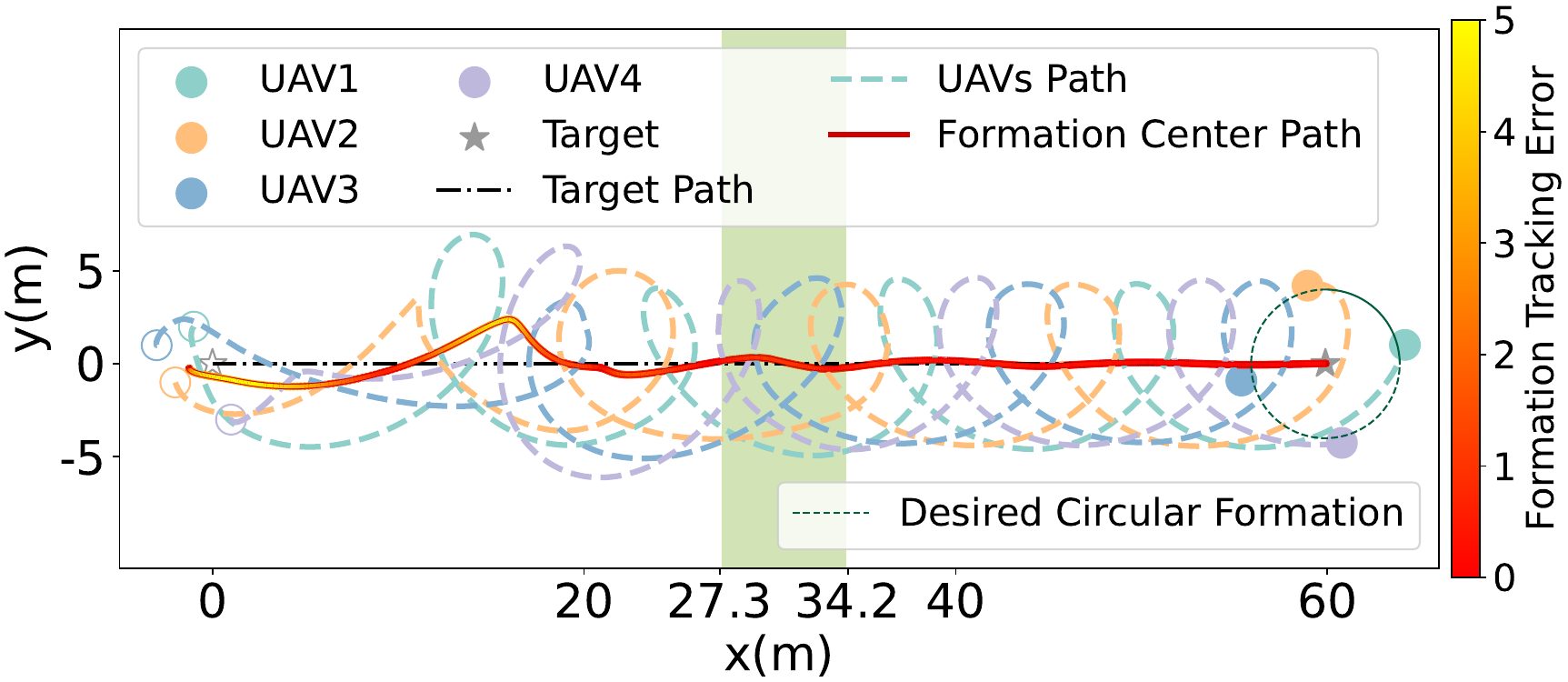}
      \caption{The paths of four UAVs enclosing and tracking a target at a constant velocity in a simulation scenario with vision occlusion. Areas with occlusions are filled in green color. }
      \label{fig:traj_block}
    \end{figure}
     \begin{figure}
              \centering
      \includegraphics[width=\linewidth]{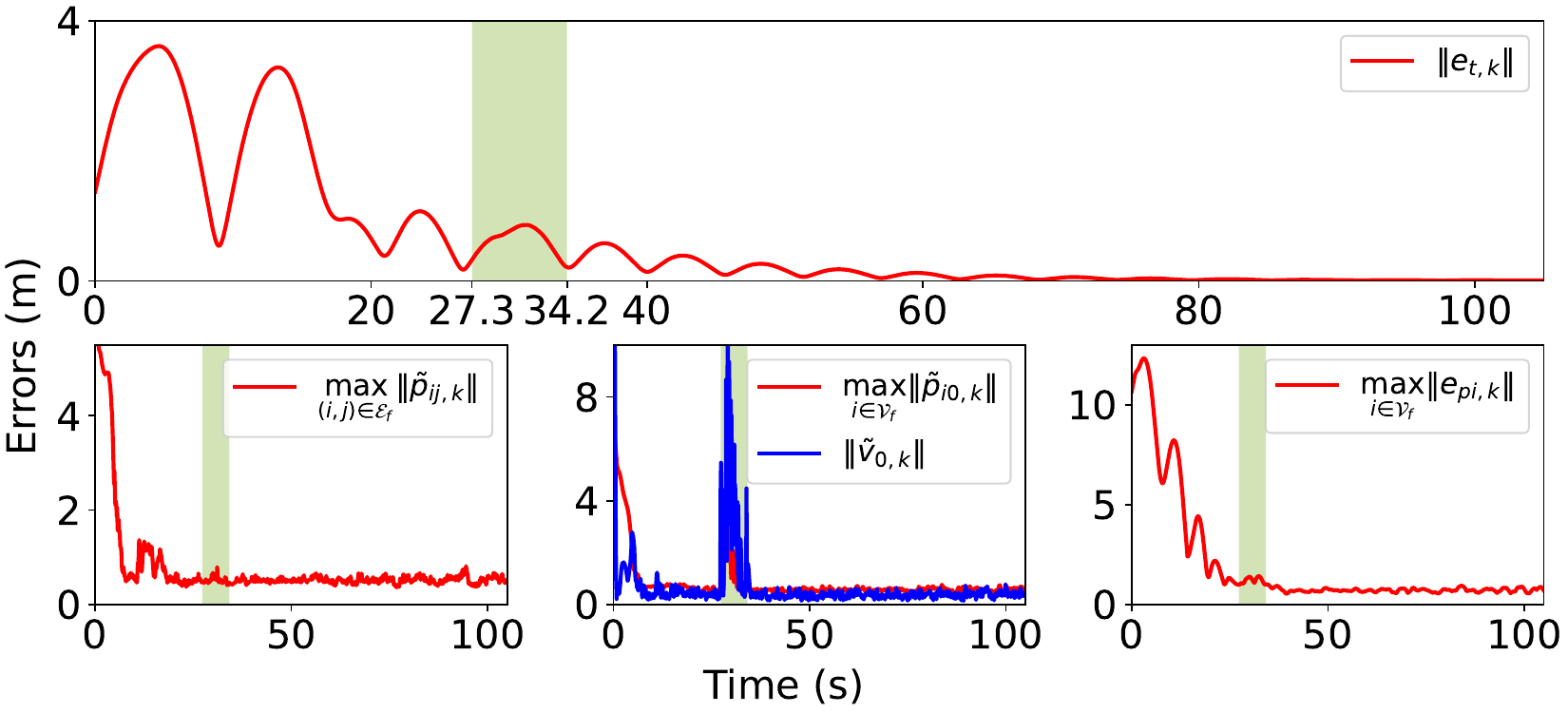}
      \caption{Error analysis of target tracking at a constant velocity in a scenario with vision occlusion. The green-filled area indicates the period when there is occlusion. }
      \label{fig:errors_blocks}
    \end{figure}
    
       \begin{figure}
              \centering		\includegraphics[width=\linewidth]{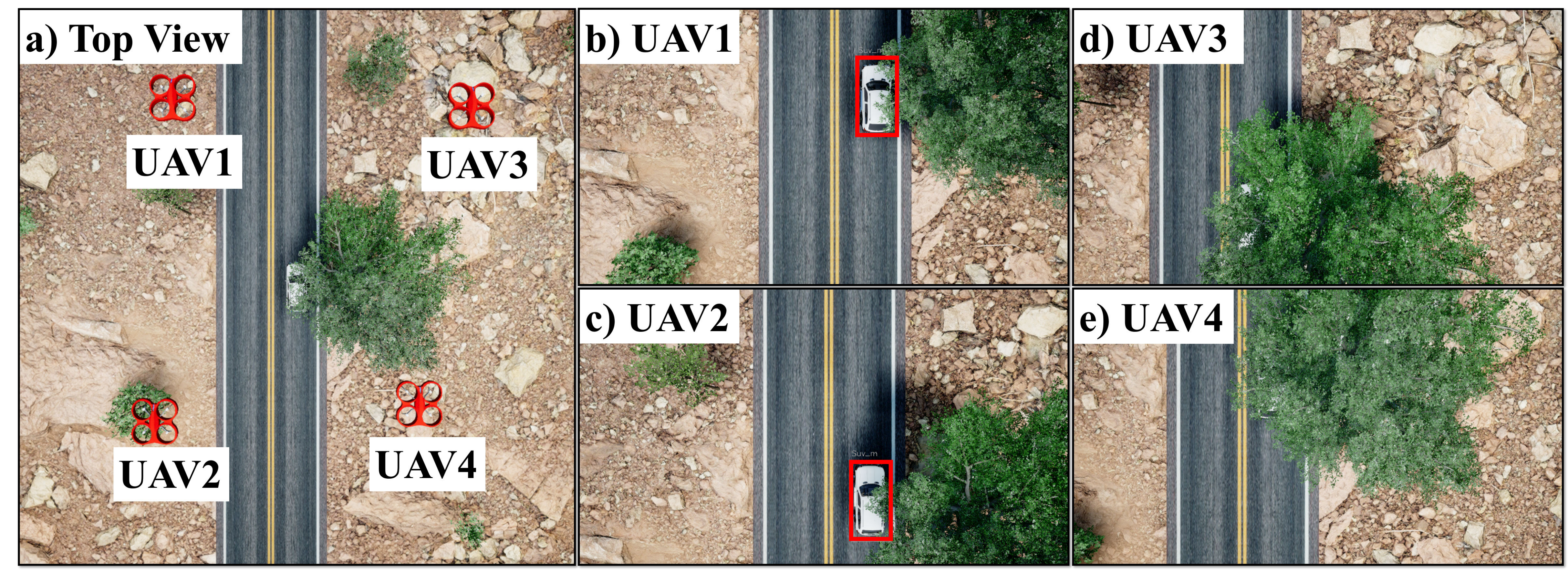}
      \caption{ At time 29.6s, only UAV 1 and 2 can see the vehicle, while UAV 3 and 4 are obstructed by trees and cannot view the target. a) The top view of the simulation. b), c), d), and e) show images captured by UAVs, respectively.}
      \label{fig:tracking}
    \end{figure}
  
  \subsection{Experimental Results}
  A vector field-based collision-free control term is applied to ensure safe flight during experiments. It is assumed that the UAVs have the omnidirectional range perception ability. Thus, UAVs can get a series of discrete points $  \boldsymbol{r}_{im,k}, m={1,...,M} \in \mathbb{N}$ around it with the distance $ \| \boldsymbol{r}_{im} \| \leq d_{safe} $, where $ \boldsymbol{r}_{im,k} $ is the vector points to the point $m$ in the local coordinate frame of UAVs. Thus, the collision-free control term $ \boldsymbol{u}_{i,k}^{safe} $ is given as 
        \begin{equation*}\label{eq:collision-free}
            \boldsymbol{u}_{i,k}^{safe} = \sum_{m=1}^M k_1 \exp \left(-\frac{\left\lVert \boldsymbol{r}_{im,k} \right\rVert^2 }{k_2}\right) \frac{\boldsymbol{r}_{im,k}}{\left\lVert \boldsymbol{r}_{im,k} \right\rVert}
      \end{equation*}
where $k_1,k_2\in \mathbb{R}$ are positive scalar parameter related to the safe distance $d_{safe}$.
 Hence, in the experiments, 
         \begin{equation*}\label{eq:controller+collision-free}
            \boldsymbol{u}_{i,k} = \pi_{U_{trac}}\left( \boldsymbol{u}_{i,k}^{trac} \right) + \pi_{U_{safe}}\left( \boldsymbol{u}_{i,k}^{safe} \right) + \boldsymbol{u}_{i,k}^*
      \end{equation*}
where $ U_{safe} >0$ is the upper bound for the collision-free term. 

     \begin{figure}
              \centering
      \includegraphics[width=0.9\linewidth]{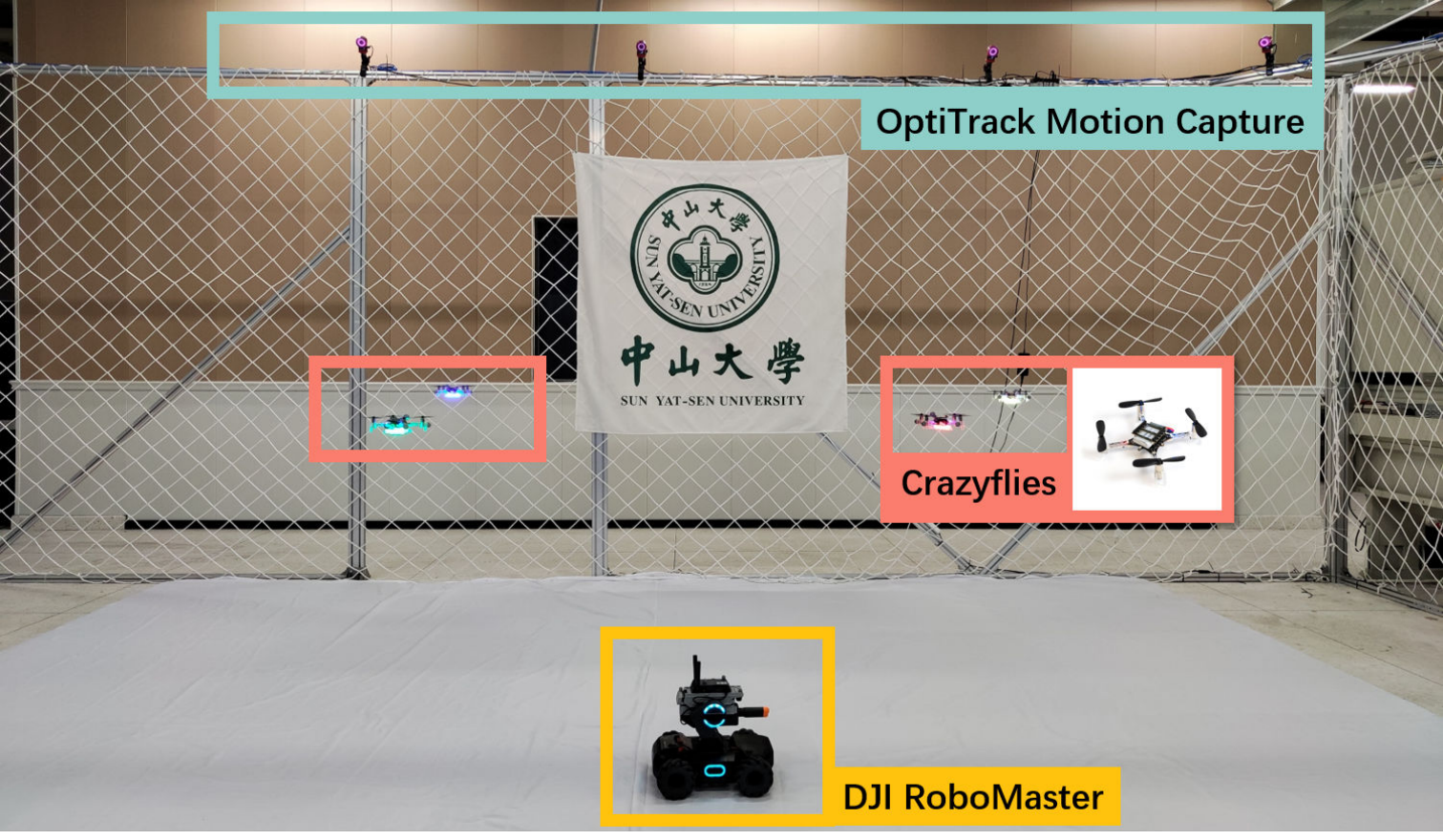}
      \caption{The diagram of the indoor experimental environment.}
      \label{fig:exp}
    \end{figure}

  By applying the collision-free control term, three experiments of four Crazyflies\footnote{\url{https://www.bitcraze.io/}} quadrotors tracking a DJI RoboMaster\footnote{\url{https://www.robomaster.com/}} ground robot are implemented in an indoor area (Fig.~\ref{fig:exp}). All the measurements are indirectly obtained by the OptiTrack motion capture system. Initially, the UAVs are hovering at $1m$ above the ground. Let $ U_{safe}=0.5 $, $d_{safe}=0.55m$ and $k_1=30,k_2=5$. The expected radius of the circular formation is chosen as $\rho=0.4m$. 

   1) \emph{Target at a Constant Velocity:} The DJI RoboMaster is moving at a constant velocity $\boldsymbol{v}_{0} = [0.1,0]^T m/s $, starting from an initial position of $ [-1.5,0]^Tm$. The initial positions of Crazyflies projection to $x,y$-plane are $[-0.7,0.5]^T,[-1.7,1]^T,[-2.3,0.3]^T$ and $[-1,-1]^Tm$. The experimental process is shown in Fig.~\ref{fig:constant_EXP_video}. As shown in Fig.~\ref{fig:errors_CONSTANT_EXP}, due to measurement noise and process noise, the relative position estimation errors between the UAVs converge to a bounded value. This further prevents the formation error from converging to zero. As a result, in Fig.~\ref{fig:traj_CONSTANT_EXP}, there is a deviation between the actual UAV positions and the desired circular formation. Due to the symmetry of the circular formation, the formation tracking error exhibits favorable convergence, ensuring that the target is enclosed at the center of the formation.

     \begin{figure}
              \centering
      \includegraphics[width=0.9\linewidth]{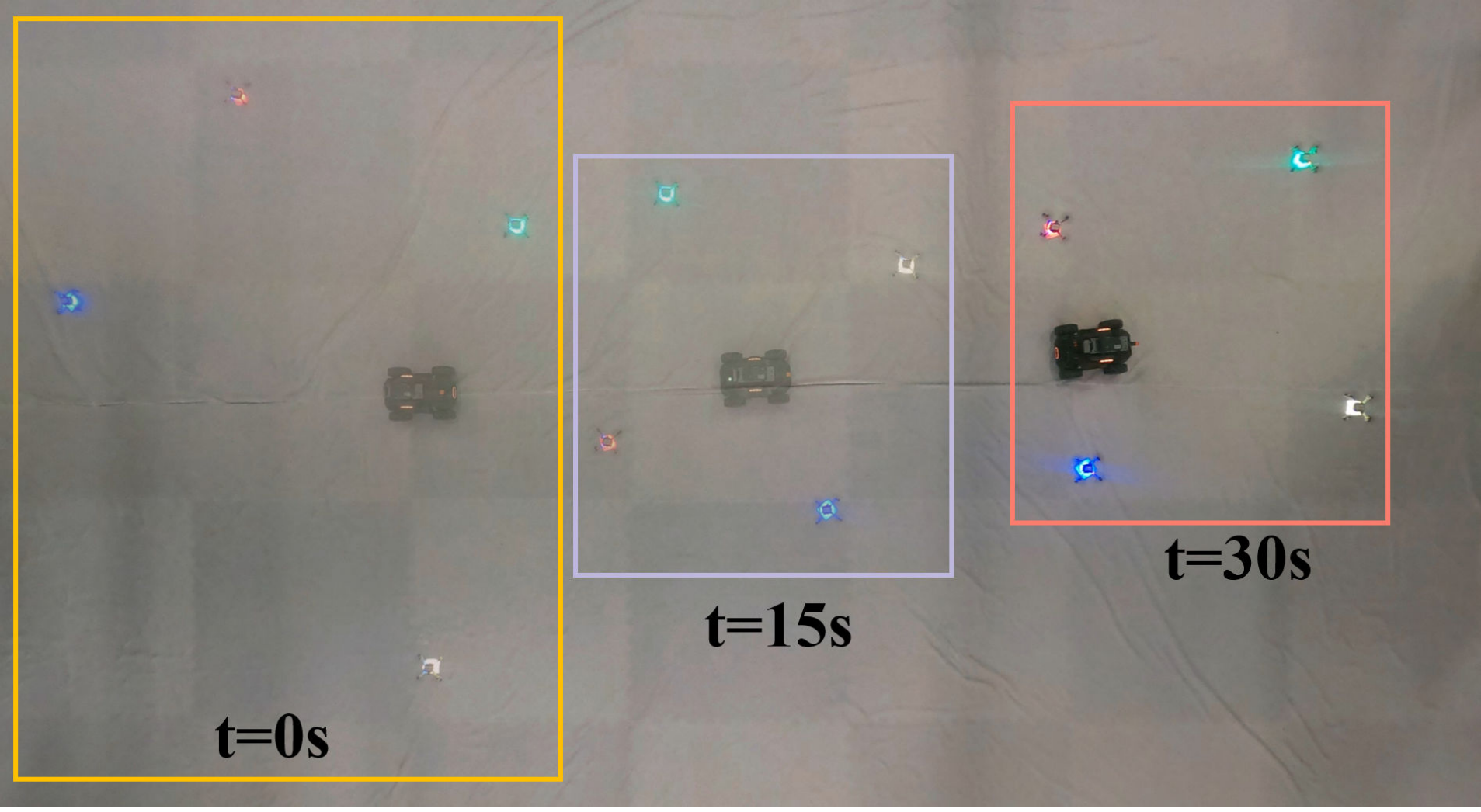}
      \caption{Experiment of four Crazyflies tracking the DJI RoboMaster moving at a constant velocity.}
      \label{fig:constant_EXP_video}
    \end{figure}

        \begin{figure}
              \centering
      \includegraphics[width=\linewidth]{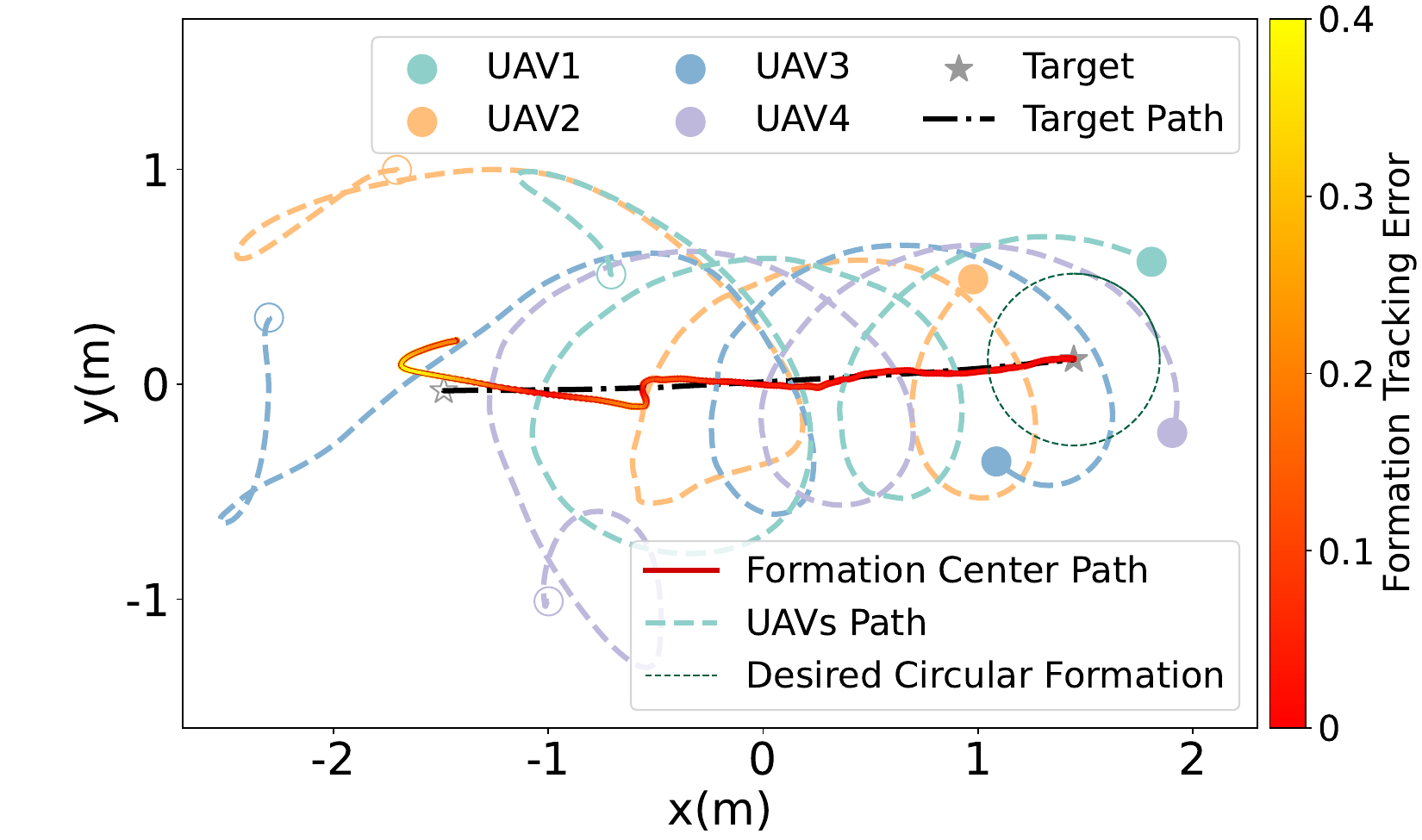}
      \caption{The paths from the experiment tracking a target with constant velocity.}
      \label{fig:traj_CONSTANT_EXP}
    \end{figure}
    
     \begin{figure}
              \centering
      \includegraphics[width=\linewidth]{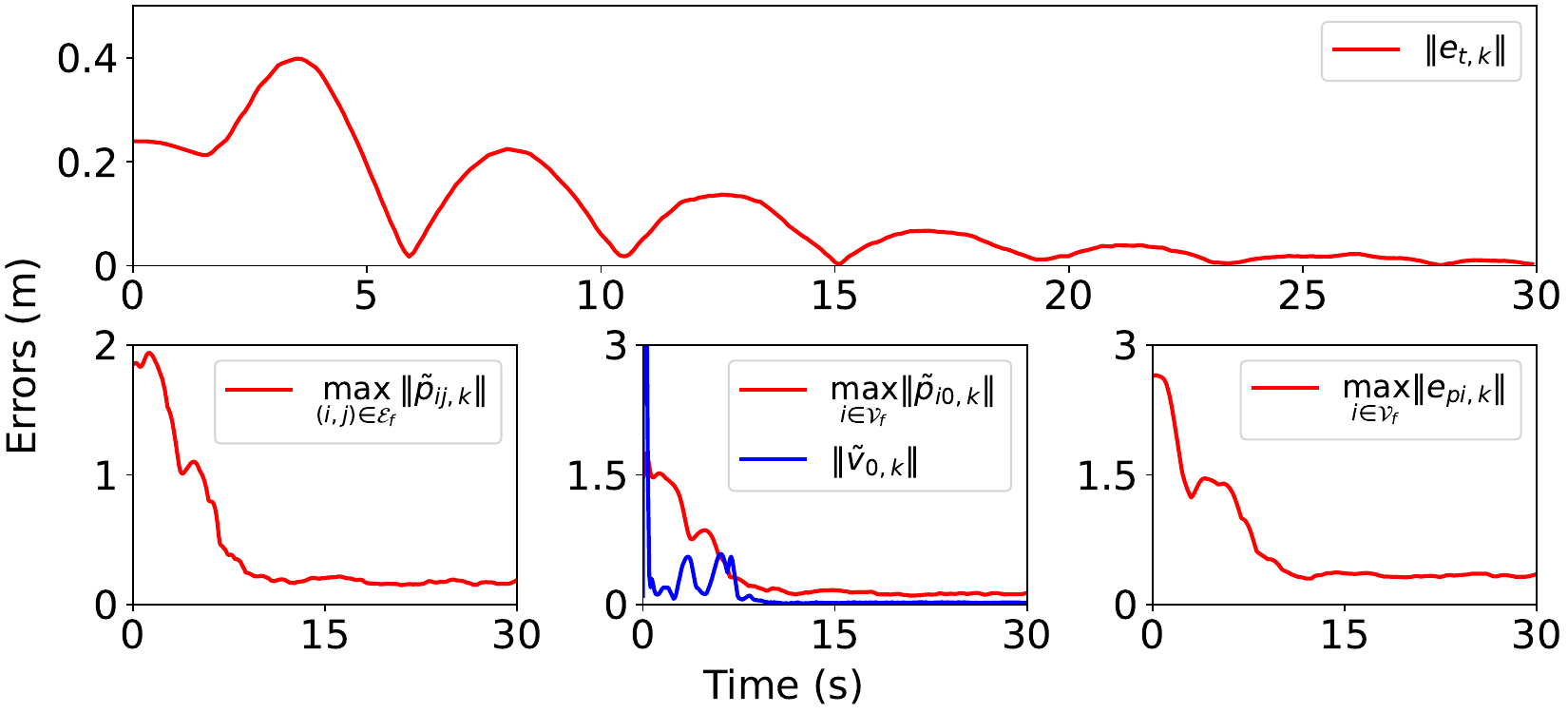}
      \caption{Error analysis of the experiment tracking a target with constant velocity.}
      \label{fig:errors_CONSTANT_EXP}
    \end{figure}

 2) \emph{Target at a Varying Velocity:} The DJI RoboMaster travels around a circle trajectory with a radius of $1m$ at a speed of $0.1m/s$ from the initial position $ [-1,0]^Tm$. The initial positions of Crazyflies projection to $x,y$-plane are $[-1.5,1]^T,[-1,-1]^T,[-2,0.3]^T$ and $[-0.5,0.5]^Tm$. The experimental process is shown in Fig.~\ref{fig:varying_EXP_video}. Fig.~\ref{fig:traj_VARYING_EXP} gives the trajectories of the Crazyflies and the DJI RoboMaster. The tracking and estimation errors are given in Fig.~\ref{fig:errors_VARYING_EXP}. 
     \begin{figure}
              \centering
      \includegraphics[width=0.85\linewidth]{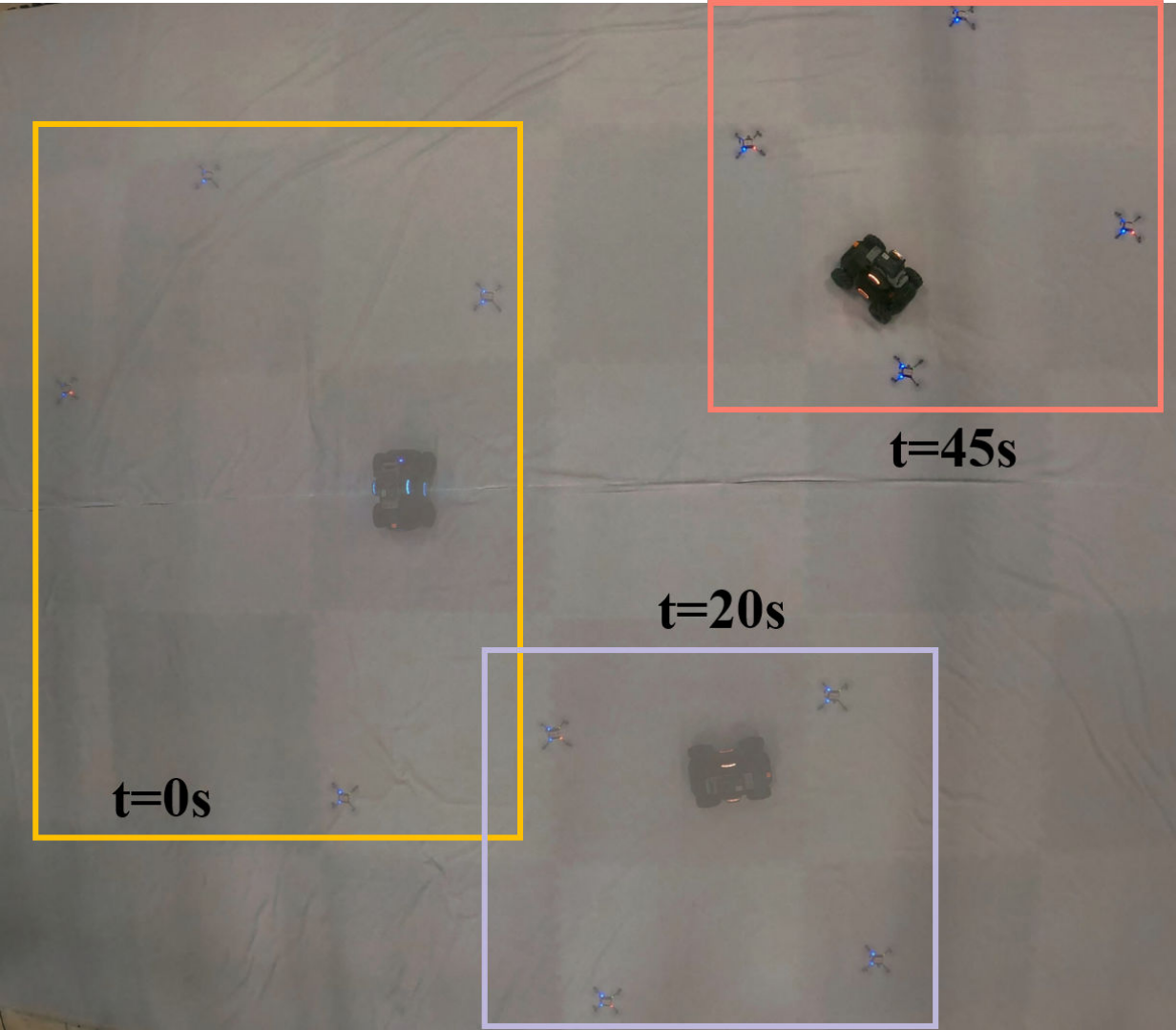}
      \caption{Experiment of four Crazyflies tracking the DJI RoboMaster moving at a varying velocity.}
      \label{fig:varying_EXP_video}
    \end{figure}
    
        \begin{figure}
              \centering
      \includegraphics[width=0.9\linewidth]{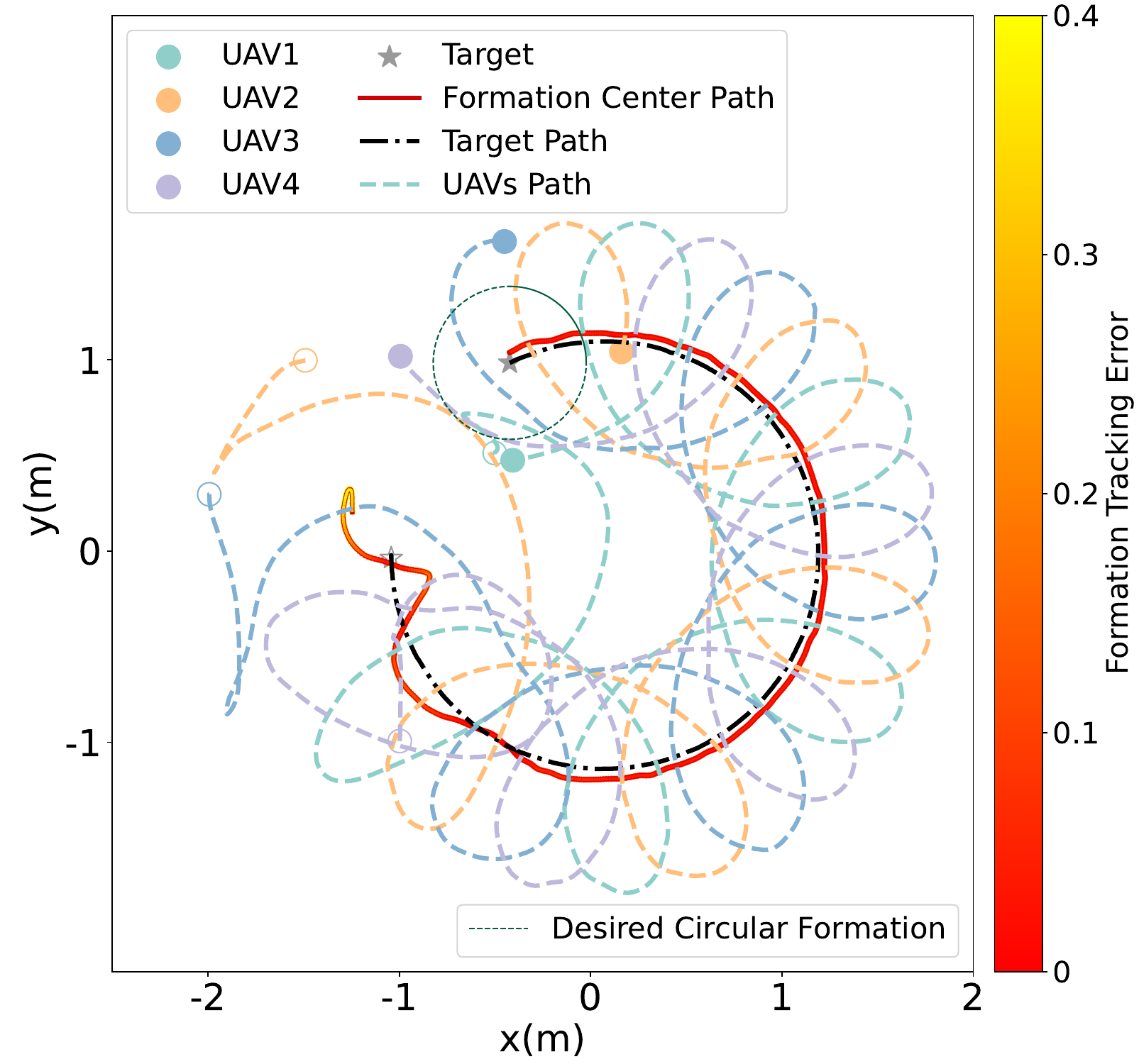}
      \caption{The paths from the experiment tracking a target with varying velocity.}
      \label{fig:traj_VARYING_EXP}
    \end{figure}
    
     \begin{figure}
              \centering
      \includegraphics[width=\linewidth]{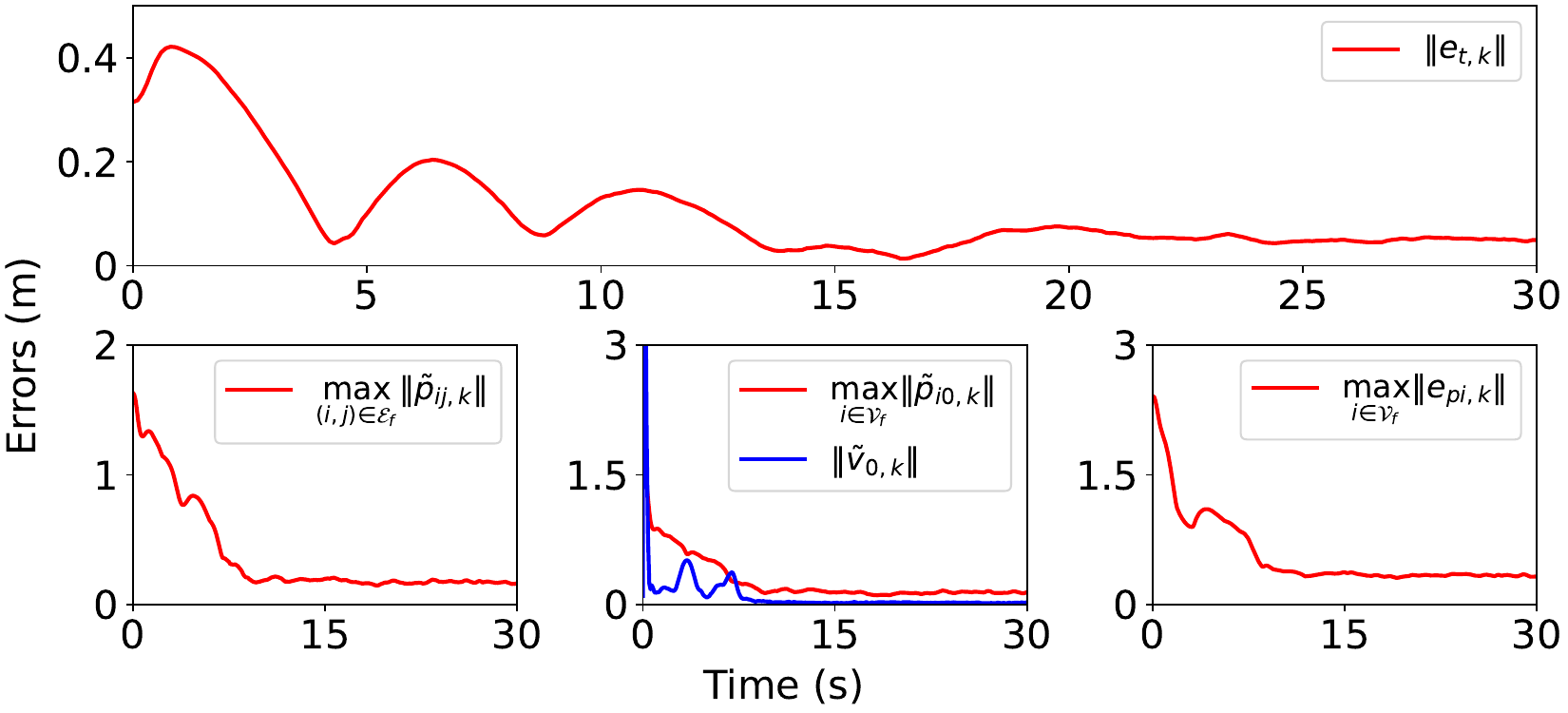}
      \caption{Error analysis of the experiment tracking a target with varying velocity.}
      \label{fig:errors_VARYING_EXP}
    \end{figure}

    3) \emph{Vision Occlusion:} In this experiment, a cylindrical obstacle is placed at $[-1.5,0.5]^Tm$, with radius of $0.2m$. The DJI RoboMaster is moving at a constant velocity $\boldsymbol{v}_{0} = [0.1,0]^T m/s $ and the initial position is $ [-1.5,0]^Tm$. The initial positions of Crazyflies projection to $x,y$-plane are $[-0.7,0.5]^T,[-1.7,1]^T,[-2.3,0.3]^T$ and $[-1,-1]^Tm$. The experimental process is shown in Fig.~\ref{fig:constant_obs_EXP_video}. As shown in Fig.~\ref{fig:traj_CONSTANT_OBS_EXP}, UAV $2$ lost the sensing of the target during $0$--$5.1s$. The green lines indicate that the cylinder blocks the line of sight.
     \begin{figure}
              \centering
      \includegraphics[width=0.9\linewidth]{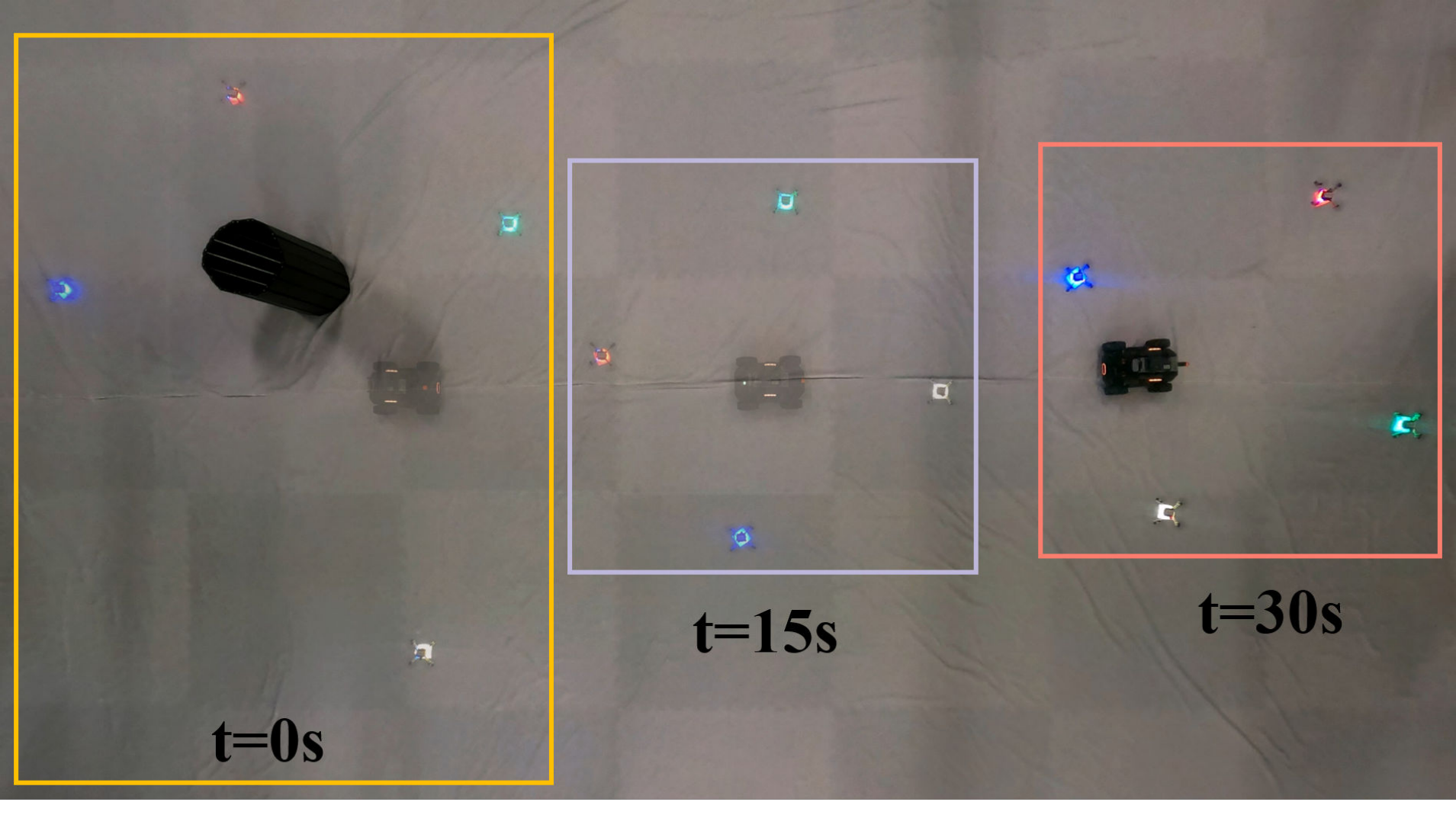}
      \caption{Experiment of four Crazyflies tracking the DJI RoboMaster moving at a constant velocity with vision occlusion..}
      \label{fig:constant_obs_EXP_video}
    \end{figure}
        \begin{figure}
              \centering
      \includegraphics[width=0.9\linewidth]{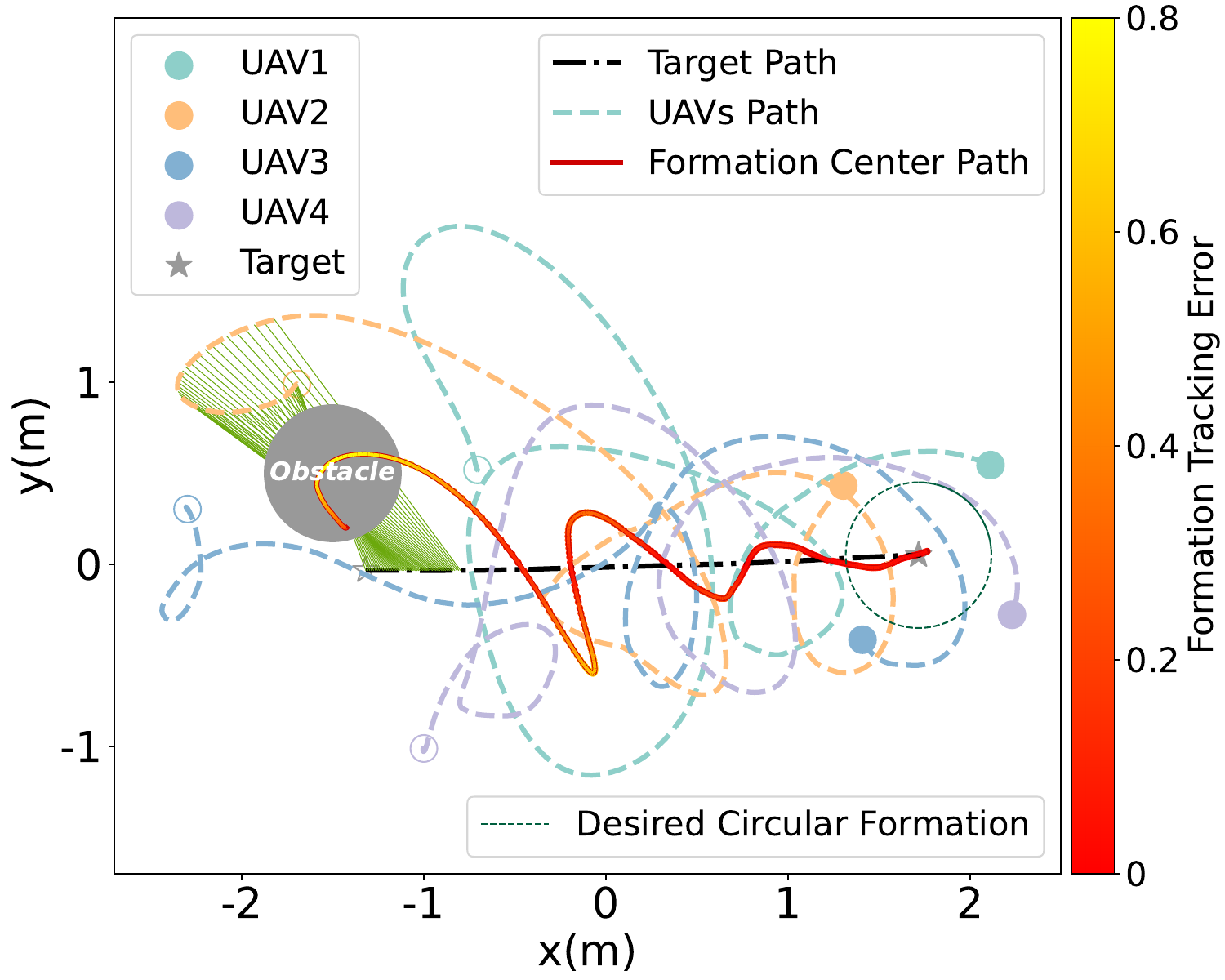}
      \caption{The paths of the experiment of the target at a constant velocity in the scenarios with vision occlusion. The green lines between the UAV $2$ and the target indicate that the FoV is blocked.}
      \label{fig:traj_CONSTANT_OBS_EXP}
    \end{figure}
    
     \begin{figure}
              \centering
      \includegraphics[width=\linewidth]{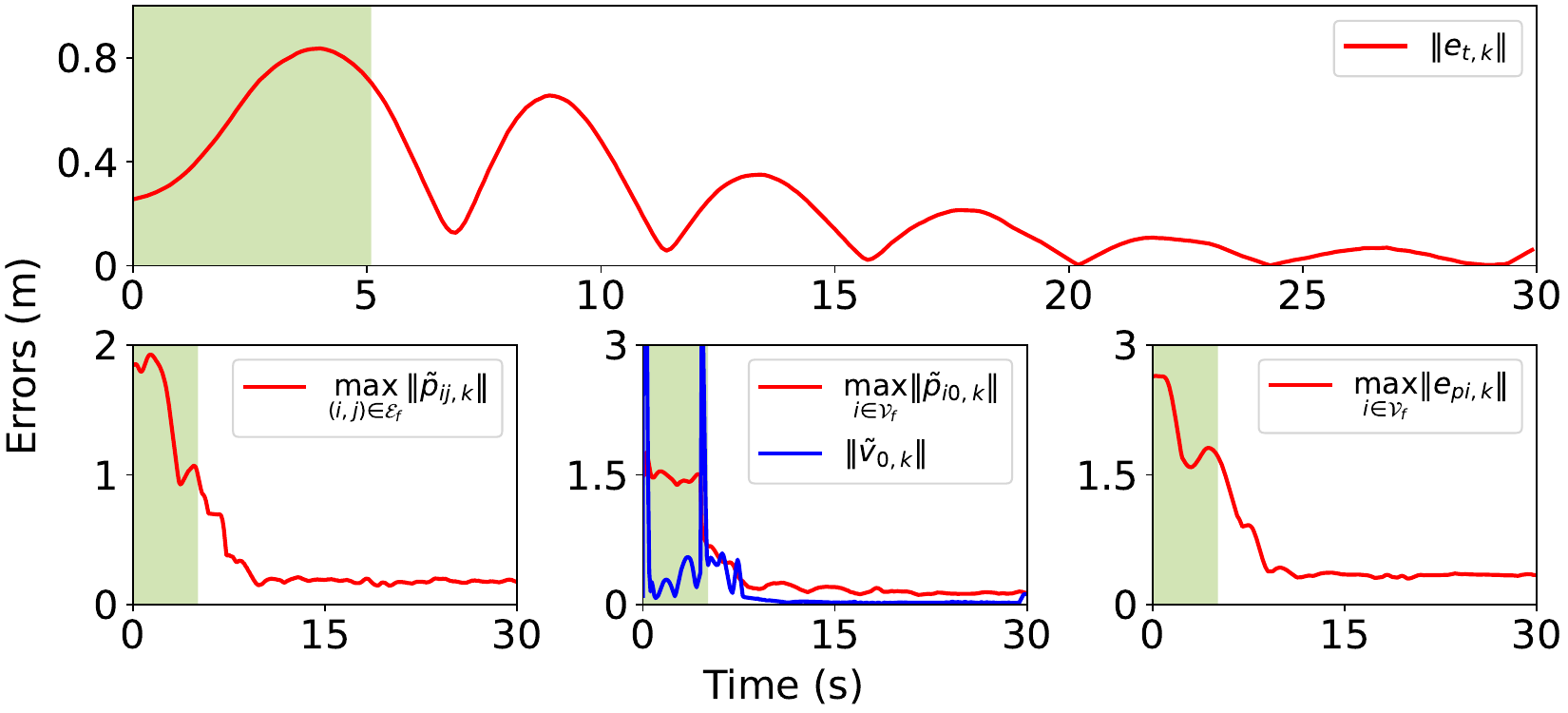}
      \caption{ Error analysis of the experiment of the target at a constant velocity in the scenarios with vision occlusion. The green-filled area indicates the period when UAV $2$ lost measurement with the target.}
      \label{fig:errors_CONSTANT_OBS_EXP}
    \end{figure}

  \section{CONCLUSION} \label{conclu}
  In this paper, coordinated formation control for enclosing and tracking a moving target in the absence of an external positioning system is studied. To this end, a relative state estimation framework is proposed by integrating RLSE and DKF methods. After that, a coupled oscillator model-based method was proposed to realize the desired time-varying formation design. Finally, a consensus-based formation controller is presented, which simultaneously guarantees the formation of enclosing and tracking tasks and ensures the convergence of the estimators. The proposed methods were not only suitable for targets moving at a constant velocity but also had a certain tracking ability for maneuvering targets. However, this paper only considers the target moving on flat terrain. The problem is more complicated by the target moving over rough terrain but is more common in rescue situations. Therefore, a promising direction is to consider the term of terrain following in the formation enclosing and tracking task. In addition, the multi-target tracking task will also be a potential research direction in the future.

\bibliography{Reference/2024TAES.bib, Reference/2024TAES_add.bib,Reference/refer} %

\begin{thebibliography}{10}
\providecommand{\url}[1]{#1}
\csname url@samestyle\endcsname
\providecommand{\newblock}{\relax}
\providecommand{\bibinfo}[2]{#2}
\providecommand{\BIBentrySTDinterwordspacing}{\spaceskip=0pt\relax}
\providecommand{\BIBentryALTinterwordstretchfactor}{4}
\providecommand{\BIBentryALTinterwordspacing}{\spaceskip=\fontdimen2\font plus
\BIBentryALTinterwordstretchfactor\fontdimen3\font minus
  \fontdimen4\font\relax}
\providecommand{\BIBforeignlanguage}[2]{{%
\expandafter\ifx\csname l@#1\endcsname\relax
\typeout{** WARNING: IEEEtran.bst: No hyphenation pattern has been}%
\typeout{** loaded for the language `#1'. Using the pattern for}%
\typeout{** the default language instead.}%
\else
\language=\csname l@#1\endcsname
\fi
#2}}
\providecommand{\BIBdecl}{\relax}
\BIBdecl

\bibitem{chungSurveyAerialSwarm2018}
S.-J. Chung, A.~A. Paranjape, P.~Dames, S.~Shen, and V.~Kumar, ``A {{Survey}}
  on {{Aerial Swarm Robotics}},'' \emph{IEEE Trans. Robot.}, vol.~34, no.~4,
  pp. 837--855, Aug. 2018.

\bibitem{zhang2017aerodynamics}
Q.~Zhang and H.~H. Liu, ``Aerodynamics modeling and analysis of close formation
  flight,'' \emph{Journal of Aircraft}, vol.~54, no.~6, pp. 2192--2204, 2017.

\bibitem{zhang2021robust}
------, ``Robust nonlinear close formation control of multiple fixed-wing
  aircraft,'' \emph{Journal of Guidance, Control, and Dynamics}, vol.~44,
  no.~3, pp. 572--586, 2021.

\bibitem{DYZhang_IROS2023}
D.~Zhang, X.~Zhang, Z.~Zhang, B.~Zhu, and Q.~Zhang, ``Reinforced potential
  field for multi-robot motion planning in cluttered environments,'' in
  \emph{2023 IEEE/RSJ International Conference on Intelligent Robots and
  Systems (IROS)}.\hskip 1em plus 0.5em minus 0.4em\relax Detroit, MI, USA:
  IEEE, Oct. 2023, pp. 699--704.

\bibitem{CHYu_ICRA2024}
C.~Yu, D.~Zhang, and Q.~Zhang, ``Grf-based predictive flocking control with
  dynamic pattern formation,'' in \emph{2024 IEEE International Conference on
  Robotics and Automation (ICRA)}.\hskip 1em plus 0.5em minus 0.4em\relax
  Yokohama, Japan: IEEE, May 2024.

\bibitem{rajaEfficientFormationControl2021}
G.~Raja, Y.~Baskar, P.~Dhanasekaran, R.~Nawaz, and K.~Yu, ``An {{Efficient
  Formation Control}} mechanism for {{Multi-UAV Navigation}} in {{Remote
  Surveillance}},'' in \emph{2021 {{IEEE Globecom Workshops}} ({{GC
  Wkshps}})}.\hskip 1em plus 0.5em minus 0.4em\relax Madrid, Spain: IEEE, Dec.
  2021, pp. 1--6.

\bibitem{bayramTrackingWildlifeMultiple2017}
H.~Bayram, N.~Stefas, K.~S. Engin, and V.~Isler, ``Tracking wildlife with
  multiple {{UAVs}}: {{System}} design, safety and field experiments,'' in
  \emph{2017 {{International Symposium}} on {{Multi-Robot}} and {{Multi-Agent
  Systems}} ({{MRS}})}.\hskip 1em plus 0.5em minus 0.4em\relax Los Angeles, CA:
  IEEE, Dec. 2017, pp. 97--103.

\bibitem{kratkyAutonomousAerialFilming2021}
V.~Kratky, A.~Alcantara, J.~Capitan, P.~Stepan, M.~Saska, and A.~Ollero,
  ``Autonomous {{Aerial Filming With Distributed Lighting}} by a {{Team}} of
  {{Unmanned Aerial Vehicles}},'' \emph{IEEE Robot. Autom. Lett.}, vol.~6,
  no.~4, pp. 7580--7587, Oct. 2021.

\bibitem{ZZhang_RAL2024}
Z.~Zhang, D.~Zhang, Q.~Zhang, W.~Pan, and T.~Hu, ``Dacoop-a: Decentralized
  adaptive cooperative pursuit via attention,'' \emph{IEEE Robotics and
  Automation Letters}, vol.~9, no.~6, pp. 5504--5511, Jun. 2024.

\bibitem{ZZhang_TNNLS2025}
Z.~Zhang, Q.~Zhang, B.~Zhu, X.~Wang, and T.~Hu, ``Easpace: Enhanced action
  space for policy transfer,'' \emph{IEEE Transactions on Neural Networks and
  Learning Systems}, vol.~36, no.~1, pp. 1272--1286, Jan. 2025.

\bibitem{litimeinSurveyTechniquesCircular2021}
H.~Litimein, Z.-Y. Huang, and A.~Hamza, ``A {{Survey}} on {{Techniques}} in the
  {{Circular Formation}} of {{Multi-Agent Systems}},'' \emph{Electronics},
  vol.~10, no.~23, p. 2959, Nov. 2021.

\bibitem{brinon-arranzCooperativeControlDesign2014}
L.~{Brinon-Arranz}, A.~Seuret, and C.~{Canudas-de-Wit}, ``Cooperative {{Control
  Design}} for {{Time-Varying Formations}} of {{Multi-Agent Systems}},''
  \emph{IEEE Trans. Automat. Contr.}, vol.~59, no.~8, pp. 2283--2288, Aug.
  2014.

\bibitem{liDistributedKalmanFilter2020}
W.~Li, Y.~Jia, and J.~Du, ``Distributed {{Kalman Filter}} for {{Cooperative
  Localization With Integrated Measurements}},'' \emph{IEEE Trans. Aerosp.
  Electron. Syst.}, vol.~56, no.~4, pp. 3302--3310, Aug. 2020.

\bibitem{liDistributedKalmanFilter2021}
W.~Li, K.~Xiong, Y.~Jia, and J.~Du, ``Distributed {{Kalman Filter}} for
  {{Multitarget Tracking Systems With Coupled Measurements}},'' \emph{IEEE
  Trans. Syst. Man Cybern, Syst.}, vol.~51, no.~10, pp. 6599--6604, Oct. 2021.

\bibitem{olfati-saberCollaborativeTargetTracking2011}
R.~{Olfati-Saber} and P.~Jalalkamali, ``Collaborative target tracking using
  distributed {{Kalman}} filtering on mobile sensor networks,'' in
  \emph{Proceedings of the 2011 {{American Control Conference}}}.\hskip 1em
  plus 0.5em minus 0.4em\relax San Francisco, CA: IEEE, Jun. 2011, pp.
  1100--1105.

\bibitem{lianDistributedKalmanConsensus2022}
B.~Lian, Y.~Wan, Y.~Zhang, M.~Liu, F.~L. Lewis, and T.~Chai, ``Distributed
  {{Kalman Consensus Filter}} for {{Estimation With Moving Targets}},''
  \emph{IEEE Trans. Cybern.}, vol.~52, no.~6, pp. 5242--5254, Jun. 2022.

\bibitem{wangCooperativeTargetTracking2012}
Z.~Wang and D.~Gu, ``Cooperative {{Target Tracking Control}} of {{Multiple
  Robots}},'' \emph{IEEE Trans. Ind. Electron.}, vol.~59, no.~8, pp.
  3232--3240, Aug. 2012.

\bibitem{olfati-saberDistributedKalmanFiltering2007}
R.~{Olfati-Saber}, ``Distributed {{Kalman}} filtering for sensor networks,'' in
  \emph{2007 46th {{IEEE Conference}} on {{Decision}} and {{Control}}}.\hskip
  1em plus 0.5em minus 0.4em\relax New Orleans, LA, USA: IEEE, 2007, pp.
  5492--5498.

\bibitem{liSelfLocalizationDistributedSystems2022}
Y.~Li and S.~Li, ``Self-{{Localization}} for {{Distributed Systems}} based on
  the {{Relative Measurement Information}},'' in \emph{2022 41st {{Chinese
  Control Conference}} ({{CCC}})}.\hskip 1em plus 0.5em minus 0.4em\relax
  Hefei, China: IEEE, Jul. 2022, pp. 4778--4783.

\bibitem{doostmohammadianDistributedEstimationApproach2022}
M.~Doostmohammadian, A.~Taghieh, and H.~Zarrabi, ``Distributed {{Estimation
  Approach}} for {{Tracking}} a {{Mobile Target}} via {{Formation}} of
  {{UAVs}},'' \emph{IEEE Trans. Automat. Sci. Eng.}, vol.~19, no.~4, pp.
  3765--3776, Oct. 2022.

\bibitem{xuDistributedPseudolinearEstimation2017}
S.~Xu, K.~Do{\u g}an{\c c}ay, and H.~Hmam, ``Distributed pseudolinear
  estimation and {{UAV}} path optimization for {{3D AOA}} target tracking,''
  \emph{Signal Processing}, vol. 133, pp. 64--78, Apr. 2017.

\bibitem{jainEncirclementMovingTargets2022}
P.~Jain, C.~K. Peterson, and R.~W. Beard, ``Encirclement of {{Moving Targets
  Using Noisy Range}} and {{Bearing Measurements}},'' \emph{Journal of
  Guidance, Control, and Dynamics}, vol.~45, no.~8, pp. 1399--1414, Aug. 2022.

\bibitem{dongCoordinateFreeCircumnavigationMoving2022}
F.~Dong, K.~You, L.~Xie, and Q.~Hu, ``Coordinate-{{Free Circumnavigation}} of a
  {{Moving Target Via}} a {{PD-Like Controller}},'' \emph{IEEE Trans. Aerosp.
  Electron. Syst.}, vol.~58, no.~3, pp. 2012--2025, Jun. 2022.

\bibitem{liuMovingTargetCircumnavigationUsing2023}
F.~Liu, C.~Guo, W.~Meng, R.~Su, and H.~Li, ``Moving-{{Target Circumnavigation
  Using Adaptive Neural Anti-Synchronization Control}} via {{Distance-Only
  Measurements}},'' \emph{IEEE Trans. Cybern.}, pp. 1--11, 2023.

\bibitem{liVGSwarmVisionBasedGene2023}
H.~Li, Y.~Cai, J.~Hong, P.~Xu, H.~Cheng, X.~Zhu, B.~Hu, Z.~Hao, and Z.~Fan,
  ``{{VG-Swarm}}: {{A Vision-Based Gene Regulation Network}} for {{UAVs Swarm
  Behavior Emergence}},'' \emph{IEEE Robot. Autom. Lett.}, vol.~8, no.~3, pp.
  1175--1182, Mar. 2023.

\bibitem{liuFormationControlMoving2023}
X.~Liu, K.~Liu, T.~Hu, and Q.~Zhang, ``Formation {{Control}} for {{Moving
  Target Enclosing}} via {{Relative Localization}},'' in \emph{2023 62nd {{IEEE
  Conference}} on {{Decision}} and {{Control}} ({{CDC}})}.\hskip 1em plus 0.5em
  minus 0.4em\relax Singapore, Singapore: IEEE, Dec. 2023, pp. 1400--1405.

\bibitem{yuBearingonlyCircumnavigationControl2019}
Y.~Yu, Z.~Li, X.~Wang, and L.~Shen, ``Bearing-only circumnavigation control of
  the multi-agent system around a moving target,'' \emph{IET Control Theory
  \&amp; Applications}, vol.~13, no.~17, pp. 2747--2757, Nov. 2019.

\bibitem{ibenthalLocalizationPartiallyHidden2023}
J.~Ibenthal, L.~Meyer, H.~{Piet-Lahanier}, and M.~Kieffer, ``Localization of
  {{Partially Hidden Moving Targets Using}} a {{Fleet}} of {{UAVs}} via
  {{Bounded-Error Estimation}},'' \emph{IEEE Trans. Robot.}, vol.~39, no.~6,
  pp. 4211--4229, Dec. 2023.

\bibitem{nguyenPersistentlyExcitedAdaptive2020}
T.-M. Nguyen, Z.~Qiu, T.~H. Nguyen, M.~Cao, and L.~Xie, ``Persistently
  {{Excited Adaptive Relative Localization}} and {{Time-Varying Formation}} of
  {{Robot Swarms}},'' \emph{IEEE Trans. Robot.}, vol.~36, no.~2, pp. 553--560,
  Apr. 2020.

\bibitem{guoUltraWidebandOdometryBasedCooperative2020}
K.~Guo, X.~Li, and L.~Xie, ``Ultra-{{Wideband}} and {{Odometry-Based
  Cooperative Relative Localization With Application}} to {{Multi-UAV Formation
  Control}},'' \emph{IEEE Trans. Cybern.}, vol.~50, no.~6, pp. 2590--2603, Jun.
  2020.

\bibitem{kleinIntegrationCommunicationControl2007}
D.~J. Klein, P.~Lee, K.~A. Morgansen, and T.~Javidi, ``Integration of
  {{Communication}} and {{Control Using Discrete Time Kuramoto Models}} for
  {{Multivehicle Coordination Over Broadcast Networks}},'' p.~7, 2007.

\bibitem{sepulchreStabilizationPlanarCollective2007}
R.~Sepulchre, D.~A. Paley, and N.~E. Leonard, ``Stabilization of {{Planar
  Collective Motion}}: {{All-to-All Communication}},'' \emph{IEEE Trans.
  Automat. Contr.}, vol.~52, no.~5, pp. 811--824, May 2007.

\bibitem{sepulchreStabilizationPlanarCollective2008}
------, ``Stabilization of {{Planar Collective Motion With Limited
  Communication}},'' \emph{IEEE Trans. Automat. Contr.}, vol.~53, no.~3, pp.
  706--719, Apr. 2008.

\bibitem{dorflerSynchronizationComplexNetworks2014}
F.~D{\"o}rfler and F.~Bullo, ``Synchronization in complex networks of phase
  oscillators: {{A}} survey,'' \emph{Automatica}, vol.~50, no.~6, pp.
  1539--1564, Jun. 2014.

\bibitem{jainEncirclementMovingTargets2019}
P.~Jain and C.~K. Peterson, ``Encirclement of {{Moving Targets}} using
  {{Relative Range}} and {{Bearing Measurements}},'' in \emph{2019
  {{International Conference}} on {{Unmanned Aircraft Systems}}
  ({{ICUAS}})}.\hskip 1em plus 0.5em minus 0.4em\relax Atlanta, GA, USA: IEEE,
  Jun. 2019, pp. 43--50.

\bibitem{yangObservabilityEnhancementBoresightCalibration2023}
X.~Yang, Z.~Cheng, and S.~He, ``Observability-{{Enhancement Boresight
  Calibration}} of {{Camera-IMU System}}: {{Theory}} and {{Experiments}},''
  \emph{IEEE Trans. Aerosp. Electron. Syst.}, vol.~59, no.~4, pp. 3643--3658,
  Aug. 2023.

\bibitem{liLocalizationCircumnavigationMultiple2018}
R.~Li, Y.~Shi, and Y.~Song, ``Localization and circumnavigation of multiple
  agents along an unknown target based on bearing-only measurement: {{A}} three
  dimensional solution,'' \emph{Automatica}, vol.~94, pp. 18--25, Aug. 2018.

\bibitem{shamesCircumnavigationUsingDistance2012a}
I.~Shames, S.~Dasgupta, B.~Fidan, and B.~D.~O. Anderson, ``Circumnavigation
  {{Using Distance Measurements Under Slow Drift}},'' \emph{IEEE Trans.
  Automat. Contr.}, vol.~57, no.~4, pp. 889--903, Apr. 2012.

\bibitem{eichlerClosedformSolutionOptimal2014}
A.~Eichler and H.~Werner, ``Closed-form solution for optimal convergence speed
  of multi-agent systems with discrete-time double-integrator dynamics for
  fixed weight ratios,'' \emph{Systems \& Control Letters}, vol.~71, pp. 7--13,
  Sep. 2014.

\bibitem{tangOnboardDetectionTrackingLocalization2020}
D.~Tang, Q.~Fang, L.~Shen, and T.~Hu, ``Onboard
  {{Detection-Tracking-Localization}},'' \emph{IEEE/ASME Trans. Mechatron.},
  pp. 1--1, 2020.

\bibitem{gantmakher2000theory}
F.~R. Gantmakher, \emph{The theory of matrices}.\hskip 1em plus 0.5em minus
  0.4em\relax American Mathematical Soc., 2000, vol. 131.

\bibitem{bishop2011modern}
R.~C. D. R.~H. Bishop, \emph{Modern control systems}, 2011.

\bibitem{liesen2015linear}
J.~Liesen and V.~Mehrmann, \emph{Linear algebra}.\hskip 1em plus 0.5em minus
  0.4em\relax Springer, 2015.

\end{thebibliography}
\bibliographystyle{IEEEtran} 

\end{document}